\documentclass[11pt]{article}

\usepackage[preprint]{acl}
\usepackage{fix-cm}
\usepackage{times}
\usepackage{latexsym}
\usepackage[T1]{fontenc}

\usepackage[utf8]{inputenc}
\usepackage[T1]{fontenc}    
\usepackage{hyperref}       
\usepackage{url}            
\usepackage{booktabs}       
\usepackage{amsfonts}       
\usepackage{nicefrac}       
\usepackage{microtype}      
\usepackage{xcolor}         
\definecolor{overlay}{HTML}{FF7F0E} 
\usepackage{siunitx}
\usepackage{array}
\usepackage{arydshln}
\usepackage{graphicx}
\usepackage{hyperref}
\usepackage{dsfont}      
\usepackage{bbm}         
\usepackage{multirow}
\usepackage{comment}
\usepackage{microtype}

\usepackage{xurl}
\usepackage{hyperref}

\usepackage{amsmath}
\usepackage{amssymb}
\usepackage{mathtools}
\usepackage{amsthm}
\newtheorem{theorem}{Theorem}
\newtheorem{lemma}{Lemma}
\newtheorem{proposition}{Proposition}
\newtheorem{corollary}{Corollary}
\theoremstyle{definition}
\newtheorem{definition}{Definition}
\usepackage{enumitem}
\usepackage{subcaption}
\usepackage{colortbl}
\usepackage{pgfplots}
\pgfplotsset{compat=1.18}
\usepackage[capitalize,noabbrev]{cleveref}
\usepackage[textsize=tiny]{todonotes}
\usepackage[most]{tcolorbox}
\usepackage[edges]{forest}
\usepackage{makecell}
\usepackage{adjustbox}
\usepackage{tabularx}

\usepackage{times}
\usepackage{xcolor}
\usetikzlibrary{arrows.meta,calc,positioning,shapes.geometric}

\definecolor{ink}{HTML}{263238}
\definecolor{muted}{HTML}{5F6B73}
\definecolor{linegray}{HTML}{B8C0C5}
\definecolor{papergray}{HTML}{F4F6F7}
\definecolor{targetblue}{HTML}{2F6690}
\definecolor{targetlight}{HTML}{EAF2F7}
\definecolor{draftorange}{HTML}{C87524}
\definecolor{draftlight}{HTML}{FBF0E4}

\newcommand{\exblockh}{5.6mm}

\usepackage{microtype}
\usepackage[normalem]{ulem}
\usepackage{xcolor}
\usepackage{inconsolata}

\usepackage{graphicx}
\usepackage{colortbl}

\usepackage{tikz}
\usetikzlibrary{positioning,fit,backgrounds,calc}
\tikzset{
  >=Latex,
  flow/.style={-{Latex[length=2.0mm,width=1.25mm]}, draw=ink,
               line width=0.55pt, rounded corners=2pt},
  branch/.style={flow, draw=muted},
  inputbox/.style={draw=linegray, fill=papergray, rounded corners=2pt,
                   line width=0.55pt, minimum height=7mm, align=center,
                   inner xsep=4mm},
  rootbox/.style={draw=ink, fill=ink, text=white, rounded corners=2.5pt,
                  line width=0.7pt, minimum width=35mm,
                  minimum height=8mm, align=center},
  notation/.style={draw=linegray, fill=white, rounded corners=2pt,
                   line width=0.5pt, minimum height=7mm, align=center,
                   inner xsep=3mm},
  collabcard/.style={draw=targetblue, fill=targetlight!48!white,
                     rounded corners=2.5pt, line width=0.65pt},
  trunccard/.style={draw=draftorange, fill=draftlight!55!white,
                    rounded corners=2.5pt, line width=0.65pt},
  tag/.style={draw=linegray, fill=white, rounded corners=1.5pt,
              line width=0.45pt, minimum height=5.6mm, align=center,
              inner xsep=1.5mm},
  smallpill/.style={draw=linegray, fill=white, rounded corners=2pt,
                    line width=0.45pt, minimum height=5.5mm, align=center,
                    inner xsep=1.5mm},
}
\definecolor{plotblue}{HTML}{4C78A8}
\definecolor{plotpink}{HTML}{D879A2}
\definecolor{plotorange}{HTML}{F2A23A}
\definecolor{plotpurple}{HTML}{7A3E9D}
\definecolor{plotgray}{HTML}{666666}

\PassOptionsToPackage{table}{xcolor}
\definecolor{mixrow}{RGB}{234,242,252}   
\definecolor{truncrow}{RGB}{252,242,232} 

\title{Revisiting Lossy Verification in Speculative Decoding: Mechanisms, Trade-offs, and Failure Modes}

\author{%
  \textbf{Tianyu Wang$^{1}$ \quad
  Yuxuan Zhou$^{2}$\thanks{Corresponding author.}\thanks{Project lead.} \quad
     Heng Li$^{1}$} \\
  \textbf{Wenbin Wang$^{1}$ \quad
Zikai Xiao$^{3}$ \quad
Chunrui Zheng$^{2,4}$ \quad
  Junyuan Shang$^{2}$} \\[6pt]
  $^{1}$Independent Researcher \quad
  $^{2}$Baidu Inc.  \quad
  $^{3}$Zhejiang University \quad
  $^{4}$Shandong University \quad
}

\begin{document}
\maketitle
\begin{abstract}
Speculative Decoding (SD) accelerates large language model inference by allowing a lightweight draft model to propose tokens that are subsequently verified in parallel by a larger target model. Recent approaches introduce lossy verification schemes to further improve efficiency by relaxing strict distributional matching. Yet such relaxation silently rewrites the decoding distribution, and the resulting acceleration can come at the cost of unstable, sometimes severely degraded generation quality. In this work, we present a principled analysis of the distributions induced by lossy verification methods. We show that many seemingly distinct approaches differ only superficially and can be unified into two categories: truncation-based verification and collaborative verification. We further construct a diagnostic evaluation framework across curated benchmarks. For truncation-based methods, we identify a fundamental pitfall—performance can degrade significantly compared to the true truncation sampling baseline due to distributional distortion. For collaborative verification, we reveal that well-designed relaxation principles, namely overshoot suppression and supervision quality, matter far more than the linear interpolation between draft and target. Our code is available at \url{https://github.com/ZhouYuxuanYX/Fast-HSD}.
\end{abstract}



\section{Introduction}
\label{intro}
Auto-regressive Large Language Models (LLMs) \cite{achiam2023gpt, touvron2023llama, bai2023qwen} are at the forefront of the AI revolution. Despite their strong performance, their non-parallelizable inference presents a significant efficiency bottleneck, particularly for long-context generation in test-time scaling \cite{openai2024o1, guo2025deepseek}, LLM agents \cite{yao2022react, schick2023toolformer, autogpt}, and multimodal reasoning~\cite{peng2025lmmr1empowering3blmms}. Speculative Decoding (SD) \cite{leviathan2023fast} mitigates this challenge by having a lightweight draft model generate candidate tokens, which are then verified by the target model in parallel, reducing expensive forward passes while ensuring the generated distribution matches that of the target model.


\begin{figure}[t]
\setlength{\abovecaptionskip}{3pt} 
\setlength{\belowcaptionskip}{3pt} 
    \centering
    \includegraphics[width=0.9\linewidth]{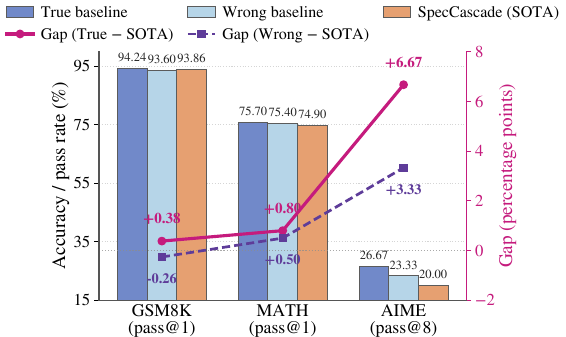}
\caption{Accuracy gap between the lossless baseline and truncation-based verification widens sharply with task difficulty, increasing from +0.38 pp on GSM8K to +6.67 pp on AIME. Here, the \textbf{True baseline} means the target model with min-p sampling, while the \textbf{Wrong baseline} indicates the target model w/o min-p sampling.}
    \label{fig:difficulty_trend}
\vspace{-15pt}
\end{figure}

While theoretical advances have pushed the limits of SD \cite{sun2024block, zhou2026overcomingjointintractabilitylossless}, the strict requirements for distribution matching continue to limit the potential for further speedup. Lossy verification methods \cite{cai2024medusa, narasimhan2024faster, zhou2024distillspec, fu2025fastlargelanguagemodel, leviathan2023fast} relax these requirements for greater acceleration, but existing work overstates their advantages under curated settings (selective hyperparameters or easy benchmarks). This obscures the true speed–quality trade-offs arising from distortion of the target distribution and leaves no clear comparison across methods, hampering the broader adoption of these methods and further advancements in the field.

In this work, we present a principled analysis of the distributions induced by lossy verification methods, resulting in a precise categorization of them. This analysis reveals that the apparent differences between methods are largely superficial, with most falling into two distinct categories. 


First, methods such as typical acceptance in Medusa~\cite{cai2024medusa} and SpecCascade~\cite{narasimhan2024faster} accept draft tokens as long as they fall within the allowed set defined by truncation-based sampling methods, specifically $\eta$-sampling~\cite{hewitt2022truncation} and min-$p$ sampling~\cite{nguyen2025turningheatminpsampling}. We refer to these as \emph{truncation-based verification} methods. Second, we show that interpolation-based, lenience-based, and judge-based relaxations essentially use weighted linear combinations of the target and draft distributions to relax the standard acceptance rule, as summarized in~\Cref{fig:taxonomy}. Following the terminology of CoS~\cite{fu2025fastlargelanguagemodel}, we group these as \emph{collaborative verification} methods.

We first examine truncation-based verification methods against their corresponding baselines, i.e., directly using the same truncation sampling strategies for the target model. Our experiments show that the seemingly comparable performance reported in prior work~\cite{cai2024medusa, narasimhan2024faster} largely arises from improved baseline performance induced by their adopted truncation sampling methods. Moreover, the performance gap between truncation-based verification and the true baseline grows with task difficulty, as shown in \Cref{fig:difficulty_trend} for SpecCascade. As harder tasks better reflect real-world scenarios, this widening gap highlights potential pitfalls of truncation-based verification in practice. This pitfall is further amplified when lossy verification methods are integrated into the EAGLE-3 speculative decoding system~\cite{li2025eagle}, where such a performance gap becomes even more pronounced.

We next turn to collaborative verification, where we identify a key principle for a better speed–quality trade-off: existing methods interpolate the draft and target distributions either uniformly in interpolation-based relaxation, adaptively in lenience-based relaxation, or through learned, context-conditioned decisions in judge-based relaxation. We find that interpolation is not what preserves generation quality: selectively suppressing the draft at overshoot tokens or a well-trained judge is sufficient.

In summary, we make these contributions:


\begin{itemize}[noitemsep, topsep=0pt, parsep=0pt, partopsep=0pt, leftmargin=*]
    \item \textbf{Unified Framework}: Unifying seemingly different lossy verification methods into two categories: truncation-based verification and collaborative verification, by reviewing their common underlying mechanisms.
    \item \textbf{Empirical Pitfalls}: Identifying a key pitfall in truncation-based verification: distributional distortion can significantly degrade performance relative to the true truncation sampling baseline.
\item \textbf{Governing Principle}:
In collaborative verification, instead of simple linear interpolation between the draft and
target probability, overshoot suppression or a well-trained judge as a surrogate acceptance rule can
significantly improve the efficiency-performance trade-off.
\end{itemize}

\definecolor{ink}{HTML}{14181C}
\definecolor{classA}{HTML}{1F4E9C}      
\definecolor{classB}{HTML}{C2600A}      
\definecolor{targetblue}{HTML}{4C78A8}  
\definecolor{draftpink}{HTML}{D65F8F}   
\definecolor{overlapgold}{HTML}{F2A541} 
\definecolor{resultink}{HTML}{14181C}   

\definecolor{condgray}{HTML}{5A6570}    
\definecolor{difflen}{HTML}{D42A2A}     
\definecolor{diffjud}{HTML}{EFB700}     
\definecolor{bandgray}{HTML}{6B7A8F}    

\DeclareRobustCommand{\cond}[1]{%
  \mbox{\hspace{1.1mm}\fontsize{5.4}{6.2}\selectfont
        \textcolor{condgray}{\textit{if}\ $#1$}}}
\DeclareRobustCommand{\note}[1]{%
  \mbox{\hspace{1.1mm}\fontsize{5.4}{6.2}\selectfont\textcolor{condgray}{$#1$}}}
\DeclareRobustCommand{\othw}{%
  \mbox{\hspace{1.1mm}\fontsize{5.4}{6.2}\selectfont
        \textcolor{condgray}{\textit{otherwise}}}}

\DeclareRobustCommand{\panelbox}[2][1.0]{%
  \tikz[baseline=-0.55ex]\fill[#2,opacity=#1,rounded corners=0.2mm]
    (0,0) rectangle (2.4mm,2.4mm);}
\DeclareRobustCommand{\panelline}{%
  \tikz[baseline=-0.55ex]{\draw[resultink,line width=0.9pt](0,0.9mm)--(4.4mm,0.9mm);
    \fill[resultink](2.2mm,0.9mm)circle(0.55mm);}}
\DeclareRobustCommand{\panelcap}{%
  \tikz[baseline=-0.55ex]\draw[resultink,line width=0.9pt](0,0.9mm)--(3.0mm,0.9mm);}
\DeclareRobustCommand{\panelband}{%
  \tikz[baseline=-0.55ex]{\fill[bandgray,opacity=0.22](0,0)rectangle(2.4mm,2.4mm);
    \draw[bandgray,opacity=0.55,line width=0.2pt](0,0)rectangle(2.4mm,2.4mm);}}
\DeclareRobustCommand{\panelmark}[1]{%
  \tikz[baseline=-0.55ex]\filldraw[fill=#1,draw=resultink,line width=0.25pt]
    (0,0.8mm) circle (1.05mm);}

\makeatletter
\newcommand{\deflen}[2]{%
  \@ifundefined{#1}{\expandafter\newlength\csname #1\endcsname}{}%
  \expandafter\setlength\csname #1\endcsname{#2}}
\makeatother
 
\deflen{cardw}{40.50mm}     
\deflen{gapin}{0.50mm}      
\deflen{gapgr}{1.00mm}      
\deflen{headaw}{122.50mm}   
\deflen{headh}{4.65mm}      
\deflen{headgap}{1.00mm}    
\deflen{cardpad}{0.40mm}    
\deflen{gtitle}{0.00mm}     
\deflen{geqn}{1.60mm}       

\newcommand{\caserow}[2]{%
  \makebox[\cardw-3.2mm][l]{$\displaystyle #1$\hfill$#2$}}

\begin{figure*}[t]
\centering
\begin{tikzpicture}[
  x=1cm, y=1cm,
  every node/.style={text=ink},
  cardtitle/.style={
    font=\fontsize{9.4}{10.4}\selectfont\bfseries,
    align=center, text width=\cardw, inner sep=0pt, outer sep=0pt},
  titleA/.style={cardtitle, text=classA},
  titleB/.style={cardtitle, text=classB},
  plotnode/.style={
    inner sep=0pt, outer sep=0pt,
    minimum width=\cardw, minimum height=23.00mm},
  eqnode/.style={
    font=\fontsize{6.8}{7.8}\selectfont,
    align=center, text width=\cardw, inner sep=0pt, outer sep=0pt},
  exnode/.style={
    font=\fontsize{6.4}{7.6}\selectfont,
    align=flush center, text width=\cardw-2mm,
    inner sep=0pt, outer sep=0pt,
    execute at begin node={\setlength{\parindent}{0pt}\setlength{\parskip}{0pt}}},,
  exA/.style={exnode, text=classA!85!black},
  exB/.style={exnode, text=classB!85!black},
  cardA/.style={
    draw=classA, fill=classA!8, rounded corners=5pt, line width=0.60pt,
    minimum width=\cardw, minimum height=45.00mm, 
    inner sep=0pt, outer sep=0pt, anchor=north west},
  cardB/.style={cardA, draw=classB, fill=classB!8},
  header/.style={
    draw=classA, fill=classA, text=white, rounded corners=5pt,
    line width=0.60pt, minimum height=\headh, anchor=south,
    font=\fontsize{8.5}{9.4}\selectfont\bfseries},
  headerB/.style={header, draw=classB, fill=classB,
    font=\fontsize{7.65}{8.4}\selectfont\bfseries}
]

\node[cardA] (card1) at (0,0) {};
\node[cardA] (card2) at ($(card1.north east)+(\gapin,0)$) {};
\node[cardA] (card3) at ($(card2.north east)+(\gapin,0)$) {};
\node[cardB] (card4) at ($(card3.north east)+(\gapgr,0)$) {};

\coordinate (exanchor) at ($(card1.south)+(0,\exblockh)$);

\node[header,  minimum width=\headaw]
  at ($(card1.north west)!0.5!(card3.north east)+(0,\headgap)$)
  {Collaborative verification};

\node[headerB, minimum width=\cardw]
  at ($(card4.north)+(0,\headgap)$)
  {Truncation-based verification};

\node[titleA, anchor=north] (title1)
  at ($(card1.north)+(0,-\cardpad)$) {Interpolation-based};

\node[plotnode, anchor=north] (plot1)
  at ($(title1.south)+(0,-\gtitle)$)
  {\includegraphics[width=\cardw, height=23mm]{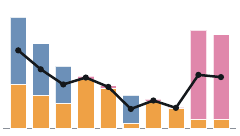}};

\node[eqnode, anchor=north] (eq1) at ($(plot1.south)+(0,-\geqn)$)
  {$\lambda p(x)+(1-\lambda)q(x)\ \note{\lambda\in[0,1]}$};

\node[exA, anchor=north] (ex1) at (card1.south|-exanchor)
  {CoS~\citep{fu2025fastlargelanguagemodel}\\%
   DIVERSED~\citep{wang2026diversed}};
\draw[classA!35, line width=0.4pt]
  ($(card1.west|-exanchor)+(2mm,1.4mm)$) -- ($(card1.east|-exanchor)+(-2mm,1.4mm)$);

\node[titleA, anchor=north] (title2)
  at ($(card2.north)+(0,-\cardpad)$) {Lenience-based};

\node[plotnode, anchor=north] (plot2)
  at ($(title2.south)+(0,-\gtitle)$)
  {\includegraphics[width=\cardw, height=23mm]{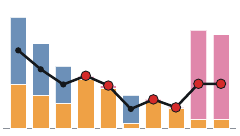}};


\node[eqnode, anchor=north] (eq2) at ($(plot2.south)+(0,-\geqn)$)
  {\adjustbox{max width=\cardw-2mm}{$\left\{\begin{array}{@{}l@{}}
      \Delta_\ell p(x)+(1-\Delta_\ell)q(x) \cond{q\le p}\\[1.5pt]
      q(x) \,\,\,\,\,\,\,\,\,\,\,\,\,\,\,\,\,\,\,\,\,\,\,\,\,\,\,\, \,\,\, \cond{p<q\le p/\ell}\\[1.5pt]
      p(x)/\ell \,\,\,\,\,\,\,\,\,\,\,\,\,\,\,\,\,\,\,\,\,\,\,\,\,\,\,\,\,\,\,\, \,\,\, \cond{q>p/\ell}
    \end{array}\right.$}};

\node[exA, anchor=north] (ex2) at (card2.south|-exanchor)
  {Lenience~\citep{leviathan2023fast}\\%
   DistillSpec~\citep{zhou2024distillspec}};
\draw[classA!35, line width=0.4pt]
  ($(card2.west|-exanchor)+(2mm,1.4mm)$) -- ($(card2.east|-exanchor)+(-2mm,1.4mm)$);

\node[titleA, anchor=north] (title3)
  at ($(card3.north)+(0,-\cardpad)$) {Judge-based};

\node[plotnode, anchor=north] (plot3)
  at ($(title3.south)+(0,-\gtitle)$)
  {\includegraphics[width=\cardw, height=23mm]{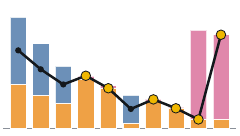}};

\node[eqnode, anchor=north] (eq3) at ($(plot3.south)+(0,-\geqn)$)
  {\adjustbox{max width=\cardw-2mm}{$\left\{\begin{array}{@{}l@{\,}l@{}}
\Delta_j p(x) + \bigl(1-\Delta_j\bigr)q(x),
    \,\, q(x)\le p(x),\\[3pt]
p(x),
    \,\,\,\,\,\,\,\,\,\,\,\,\,\,\, q(x)>p(x)\ \text{and}\ j_{\phi}(x)=0,\\[3pt]
q(x), \,\,\,\,\,\,\,\,\,\,\,\,\,\,\,
     q(x)>p(x)\ \text{and}\ j_{\phi}(x)=1.
\end{array}\right.$}};

\node[exA, anchor=north] (ex3) at (card3.south|-exanchor)
  {Judge Decoding~\citep{bachmann2025judge}\\%
   AutoJudge~\citep{garipov2025autojudge}};
\draw[classA!35, line width=0.4pt]
  ($(card3.west|-exanchor)+(2mm,1.4mm)$) -- ($(card3.east|-exanchor)+(-2mm,1.4mm)$);

\node[titleB, anchor=north] (title4)
  at ($(card4.north)+(0,-\cardpad)$) {Truncation-based};

\node[plotnode, anchor=north] (plot4)
  at ($(title4.south)+(0,-\gtitle)$)
  {\includegraphics[width=\cardw, height=23mm]{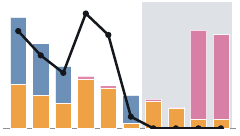}};

\node[eqnode, anchor=north] (eq4) at ($(plot4.south)+(0,-\geqn)$)
  {$\left\{\begin{array}{@{}l@{}}
      q(x)\big/\!\sum\nolimits_{v\in\mathcal{A}_{\Theta}}q(v) \cond{x\in\mathcal{A}_{\Theta}}\\[1.5pt]
      0 \,\,\,\,\,\,\,\,\,\,\,\,\,\,\,\,\,\,\,\,\,\,\,\,\,\,\,\,\,\,\,\,\,\,\,\,\,\,\,\, \cond{x\notin\mathcal{A}_{\Theta}}
    \end{array}\right.$};

\node[exB, anchor=north] (ex4) at (card4.south|-exanchor)
  {SpecCascade~\citep{narasimhan2024faster}\\%
   Typical Acceptance~\citep{cai2024medusa}};
\draw[classB!35, line width=0.4pt]
  ($(card4.west|-exanchor)+(2mm,1.4mm)$) -- ($(card4.east|-exanchor)+(-2mm,1.4mm)$);

\end{tikzpicture}

\caption{
Mechanism-level unification of lossy verification.
Each panel plots the same target/draft pair over ten tokens and the
distribution $r$ the method actually samples from.
Legend:
\panelbox[0.82]{targetblue}~target $p(x)$,
\panelbox[0.75]{draftpink}~draft $q(x)$,
\panelbox[0.95]{overlapgold}~overlap $\min(p,q)$,
\panelline~resulting distribution, and
\panelband~tokens outside the allowed set $\mathcal{A}_{\Theta}$.
Markers flag the tokens where a method departs from plain interpolation:
\panelmark{difflen}~lenience and \panelmark{diffjud}~judge.
See~\Cref{sec:detailed_derivation} for detailed derivations.
}
\label{fig:taxonomy}
\end{figure*}
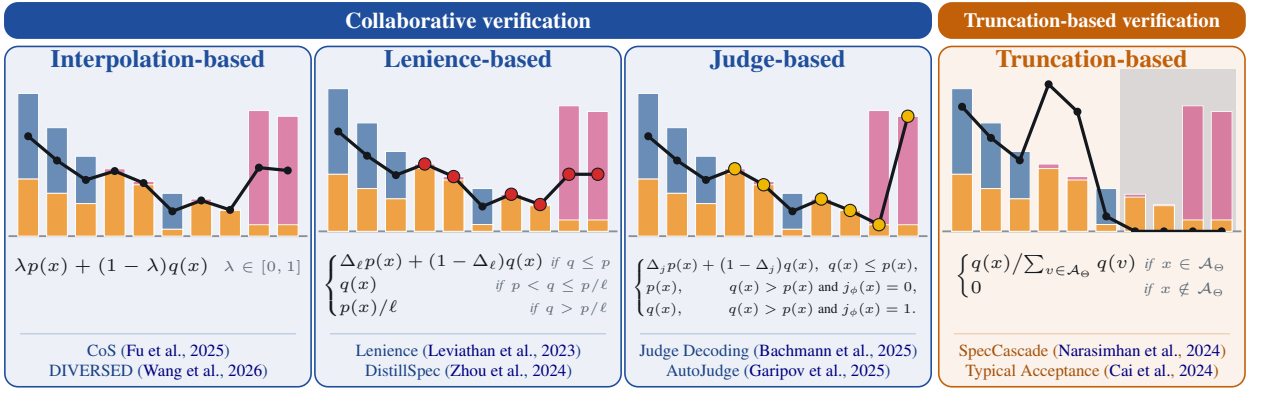

\section{Preliminaries}
\label{rel work}

We first introduce the speculative decoding framework and its lossless
verification mechanism (\Cref{sec:sd_background}), which forms the foundation
for many lossy methods. We then review two decoding techniques:
collaborative decoding (\Cref{sec:collaborative_decoding}) and truncation
sampling (\Cref{sec:truncation_sampling}). As shown in~\Cref{sec:unifying}, these techniques underlie the two categories of lossy verification.

\subsection{Speculative Decoding}
\label{sec:sd_background}
Let $p$ and $q$ denote the next-token distributions of the target and draft models, respectively, over a shared vocabulary $\mathcal{V}$. In Speculative decoding~\citep{leviathan2023fast}, each draft token $x$ sampled from $q$ is accepted with probability:
\vspace{-3pt}
\begin{equation}
  \label{eq:lossless_acc}
  h(x) = \min\!\left(1,\;\frac{p(x)}{q(x)}\right).
\end{equation}
If the draft token $x$ is rejected, a replacement token is resampled from the residual distribution
$\bigl(p(x) - \min\{p(x),\,q(x)\}\bigr)_{+}$, renormalized over $\mathcal{V}$, ensuring the generated distribution equals the target $p$ exactly.
Follow-up work increases acceptance rates by proposing multiple candidates or
tree-structured drafts~\citep{sun2023spectr, miao2024specinfer, yang2024multi, li2025eagle, fan2026flatter},
while block-level verification extends lossless guarantees to sequences of
tokens~\citep{sun2024block, zhou2026overcomingjointintractabilitylossless}.
As lossless verification approaches its theoretical limits, lossy approaches achieve
further speedup by modifying the acceptance criterion: some relax the reference
distribution in Eq.~\eqref{eq:lossless_acc}~\citep{leviathan2023fast, zhou2024distillspec,
fu2025fastlargelanguagemodel}, while others gate acceptance via a set-membership criterion
derived from truncation sampling~\citep{cai2024medusa, narasimhan2024faster}.
Despite these advances, the underlying mechanisms of lossy methods and their
speed–quality trade-offs remain poorly understood, which we address in~\Cref{sec:unifying}.

\subsection{Collaborative Decoding}
\label{sec:collaborative_decoding}

Collaborative decoding combines the predictions of multiple models to form a reshaped
next-token distribution. Using the target distribution $p$ and draft distribution $q$ as defined in \Cref{sec:sd_background}, and letting $\lambda\in[0,1]$ denote an interpolation coefficient, \textbf{Weighted Ensembling (WE)}~\cite{huang2024ensemble, yao2024determine} computes the next-token probability as a convex mixture:
\begin{equation}
\label{eq:we_dist}
  p_{\mathrm{WE}}(x) = \lambda\,p(x) + (1-\lambda)\,q(x).
\end{equation}

\textbf{Contrastive Decoding (CD)}~\cite{li2023contrastive, o2023contrastive}
instead reweights $p$ using a contrastive factor derived from $q$:
\begin{equation}
  \label{eq:cd_dist}
  p_{\mathrm{CD}}(x) \;=\;
  \frac{p(x)/q(x)^{\lambda}}{\sum_{v\in\mathcal{V}} p(v)/q(v)^{\lambda}}.
\end{equation}

Whereas standard speculative decoding aims to match the target distribution,  \textbf{Collaborative Decoding via Speculation (CoS)} \cite{fu2025fastlargelanguagemodel} instead uses one of these combined distributions as the verification target, yielding additional speedup.

\subsection{Truncation Sampling}
\label{sec:truncation_sampling}

Truncation sampling restricts the vocabulary at each decoding step to an
allowed set $\mathcal{A}_\Theta \subseteq \mathcal{V}$ determined by a
truncation strategy $\Theta$, discarding low-probability tokens to reduce the
risk of incoherent outputs. Let
$Z_\Theta(p)=\sum_{v\in\mathcal{A}_\Theta}p(v)$ denote normalization
constant. Resulting distribution is
\begin{equation}
p_\Theta(x)
=
\begin{cases}
\dfrac{p(x)}{Z_\Theta(p)}, & x \in \mathcal{A}_\Theta, \\[2pt]
0,                         & x \notin \mathcal{A}_\Theta.
\end{cases}
\label{eq:truncation-sampling}
\end{equation}

Two representative strategies, \textbf{min‑p sampling}~\cite{nguyen2025turningheatminpsampling} and \textbf{$\eta$‑sampling}~\cite{hewitt2022truncation}, are described below:

\textbf{Min-p sampling} adapts to the top candidate's confidence by applying
a dynamic cutoff relative to the peak probability, where
$p_{\text{base}} \in (0,1)$:

\begin{equation}
\mathcal{A}_{\text{min-p}}
=
\bigl\{
x \in \mathcal V :
p(x)
\ge
p_{\text{base}}
\cdot
\max_{v \in \mathcal V} p(v)
\bigr\}
\label{eq:min-p}
\end{equation}



\textbf{$\eta$-sampling} adapts to model uncertainty by removing tokens below
a threshold derived from the distribution entropy
$H(p)=-\sum_{v\in\mathcal V}p(v)\log p(v)$, where
$\varepsilon,\delta>0$:
\begin{equation}
\mathcal{A}_{\eta}
=
\bigl\{
x \in \mathcal V :
p(x)
\ge
\min\!\left(
\varepsilon,\,
\delta e^{-H(p)}
\right)
\bigr\}
\label{eq:eta}
\end{equation}

In~\Cref{sec:unifying}, we show that state-of-the-art lossy verification methods~\cite{narasimhan2024faster,cai2024medusa} essentially accept a draft token as long as it falls within the allowed set defined by one of these truncation strategies.


\section{Unifying Lossy Verification}
\label{sec:unifying}

In this section, we show that existing lossy verification methods fall into two paradigms: \textbf{collaborative verification}, which relaxes verification by matching the combined distribution from collaborative decoding, and \textbf{truncation-based verification}, which accepts draft tokens based on the allowed sets induced by truncation sampling. 


\subsection{Collaborative Verification}
\label{sec: col vrf}
Collaborative verification relaxes the acceptance criterion so that the generated distribution becomes a combination of the draft and target distributions, improving decoding speed at the cost of degraded performance. This view covers both (i) methods that explicitly construct the combined distribution~\cite{fu2025fastlargelanguagemodel}, (ii) methods whose acceptance rule implicitly induces such a distribution ~\cite{leviathan2023fast}, (iii) methods whose learned verifiers selectively retain otherwise rejected draft tokens~\cite{bachmann2025judge}.

\textbf{Interpolation-based relaxation}
modifies the acceptance probability of the lossless verification to: 
\begin{equation}
h(x)\;=\;\min\!\left(1,\frac{\lambda\,p(x) + (1-\lambda)\,q(x)}{q(x)}\right).
\end{equation}

Under this rule, the resulting token distribution becomes a convex
mixture of the target and draft distributions (see \cref{fig:taxonomy}): 
\begin{equation}
P(\text{generate}~x)=\lambda\,p(x) + (1-\lambda)\,q(x) \label{eq:we_yield}.    
\end{equation}

We show that CoS~\cite{fu2025fastlargelanguagemodel},
DIVERSED~\cite{wang2026diversed}, and
Fuzzy SD~\cite{holsman2025fuzzy} share the same
interpolation-based relaxation form and differ primarily in how the
mixture coefficient is determined (see~\Cref{sec:diversed_derivation,sec:fuzzy_derivation} for details).

\textbf{Lenience-based relaxation} modifies the acceptance rule with a lenience factor $\ell\in(0,1]$:
\begin{equation}
\label{eq:lenience_acc}
h(x)\;=\;\min\!\left(1,\frac{p(x)}{\ell\,q(x)}\right).
\end{equation}

The induced distribution admits the following decomposition (see \Cref{sec:lenience_derivation} for the proof):
\begin{equation}
\begin{gathered}
P(\mathrm{generate}\,x) =\\
\left\{\begin{array}{@{}l@{\,}l@{}}
\Delta_\ell p(x) + (1-\Delta_\ell) q(x), \,\, q(x)\le p(x),\\[3pt]
q(x), \,\,\,\,\,\,\,\,\,\,\,\,\,\,\,\,\,\,\,\,\,\,\,\,\, p(x) < q(x) \le p(x)/\ell,\\[3pt]
p(x)/\ell, \,\,\,\,\,\,\,\,\,\,\,\,\,\,\,\,\,\,\,\,\,\,\,\,\,\,\,\,\,\,\,\,\,\,\,\,\,\,\,  q(x) \ge p(x)/\ell.
\end{array}\right.
\end{gathered}
\label{eq:lenience_yield}
\end{equation}

where $\Delta_\ell = \bigl( \tfrac{1}{2}\sum_{v\in\mathcal{V}} \lvert q(v) - p(v)/\ell \rvert
+ \tfrac{1}{2} - \tfrac{1}{2\ell} \bigr) \big/ \bigl( \tfrac{1}{2}\sum_{v\in\mathcal{V}} \lvert q(v) - p(v) \rvert \bigr)$.

This decomposition reveals an adaptive correction mechanism: when the draft \emph{underestimates} the target ($q \le p$), interpolation is applied; moderate overestimation ($p<q\le p/\ell$) is left unchanged (pure $q$); and only the \emph{severe} overshoot region ($q\ge p/\ell$) is corrected (see \Cref{fig:taxonomy}).

Intuitively, the ceiling $p/\ell$ term in~\Cref{eq:lenience_yield} caps the generation probability when the draft distribution is overconfident. This ceiling prevents the draft from dominating the output at tokens where it assigns probability far beyond what the target distribution $p$ would allow, even after relaxation by the lenience factor $\ell$. Separately, $\tfrac{1}{2}\sum_{v\in\mathcal{V}}|q(v)-p(v)|$ is the total variation distance between $q$ and $p$, which measures the total mismatched probability mass between the draft and target distributions. We provide a detailed interpretation in~\Cref{sec: further_interpretation}.

We reveal that Lenience~\citep{leviathan2023fast} adopts the linear relaxation
$p/\ell$. DistillSpec~\citep{zhou2024distillspec} extends this framework
by exploring alternative relaxation forms, including $p/\ell^2$ and
$p^\ell$.

\textbf{Judge-based relaxation} relaxes standard speculative decoding with a learned judgment signal. Let \(s_{\phi}(x)\in\mathbb{R}\) denote the confidence that a learned judge assigns to draft token $x$, and let \(j_{\phi}(x)=\mathds{1}\!\left[s_{\phi}(x)\ge\delta\right]\in\{0,1\}\) indicate whether the judge predicts that $x$ should be accepted, where $\delta \in [0,1]$. Combining with~\Cref{eq:lossless_acc} gives
\begin{equation}
\label{eq:judge_acc}
\resizebox{0.85\columnwidth}{!}{$h(x)
=
\begin{cases}
1, & q(x)\leq p(x),\\[3pt]
\dfrac{p(x)}{q(x)}
+
\left(1-\dfrac{p(x)}{q(x)}\right)j_{\phi},
& q(x)>p(x).
\end{cases}$}
\end{equation}

Thus, the resulting token distribution is:
\begin{equation}
\begin{gathered}
P(\mathrm{generate}\,x)=\\
\left\{
\begin{array}{@{}l@{\,}l@{}}
\Delta_j p(x) + \bigl(1-\Delta_j\bigr)q(x),
    \,\, q(x)\le p(x),\\[3pt]
p(x),
    \,\,\,\,\,\,\,\,\,\,\,\,\,\,\, q(x)>p(x)\ \text{and}\ j_{\phi}(x)=0,\\[3pt]
q(x), \,\,\,\,\,\,\,\,\,\,\,\,\,\,\,
     q(x)>p(x)\ \text{and}\ j_{\phi}(x)=1.
\end{array}
\right.
\end{gathered}
\label{eq:judge_pgenerate}
\end{equation}
where $\Delta_j = \bigl( \sum_{v\in\mathcal{V}} \bigl(q(v)-p(v)\bigr)_+ \bigl(1-j(v)\bigr) \bigr)
\big/ \bigl( \sum_{v\in\mathcal{V}} \bigl(p(v)-q(v)\bigr)_+ \bigr)$.
see \Cref{sec:judge_derivation} for detailed derivation.

We show that Judge Decoding~\citep{bachmann2025judge} and AutoJudge~
\citep{garipov2025autojudge}
share the same judge-based relaxation form and differ in how
the judge is trained (see \Cref{app:autojudge-implementation} for detailed experiment implementation).

In summary, collaborative verification methods differ primarily in how they handle draft-model overestimation ($q(x)>p(x)$). Interpolation-based relaxation continues to blend $p(x)$ and $q(x)$ using the same form in the underestimation region. Lenience-based relaxation retains $q(x)$ under moderate overestimation but caps severe overshoot at $p(x)/\ell$. Judge-based relaxation instead uses a learned judgement signal to determine whether the overestimated draft contribution should be retained.

\begin{figure*}[t]
    \centering
    \includegraphics[width=0.9\textwidth]
        {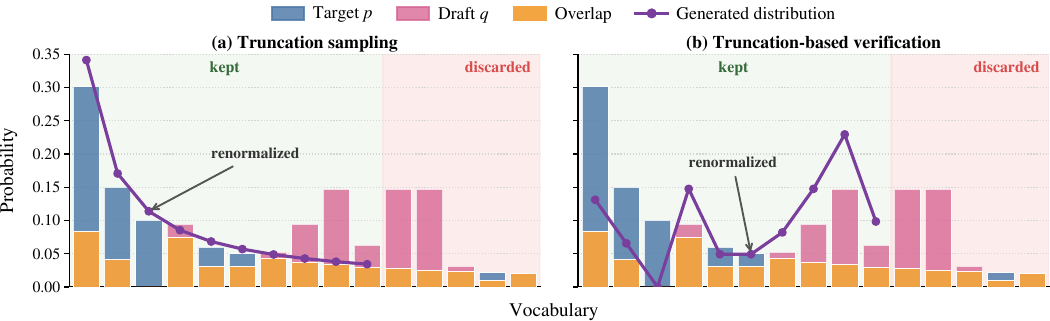}
    \caption{
    Comparison of distributions induced by truncation sampling and
    truncation-based verification. Tokens are sorted by decreasing target
    probability. Blue and pink bars denote the target distribution $p$
    and draft distribution $q$, respectively; orange denotes their
    overlap. The purple line shows the output distribution.
    }
    \label{fig:verification_mechanisms}
\end{figure*}

\subsection{Truncation-based Verification}
\label{sec:trunc-verification}
We formally define truncation-based verification. Within this framework, we distinguish two principal approaches: \textbf{SpecCascade}~\cite{narasimhan2024faster} and \textbf{typical acceptance}~\cite{cai2024medusa}.


\begin{definition}[Truncation-based verification]
\label{def:truncation-verification}
Let $\mathcal{A}_\Theta \subseteq \mathcal{V}$ denote the allowed set of a
truncation strategy $\Theta$. Truncation-based verification replaces the
acceptance rule in~\Cref{eq:lossless_acc} with
$h(x) = \mathds{1}\{x \in \mathcal{A}_\Theta\}$: a draft token is accepted if
it lies in $\mathcal{A}_\Theta$, and tokens outside this set are rejected.
\end{definition}


The resulting distribution is the renormalized draft distribution
restricted to $\mathcal{A}_\Theta$, where
$Z_\Theta(q)=\sum_{v\in\mathcal{A}_\Theta}q(v)$ is the normalization constant:
\begin{equation}
\label{eq:trunc_verif_pyield}
\resizebox{0.8\columnwidth}{!}{$
P(\mathrm{generate}\,x)=
\begin{cases}
q(x)/Z_\Theta(q), & x\in\mathcal{A}_\Theta,\\[2pt]
0, & x\notin\mathcal{A}_\Theta.
\end{cases}
$}
\end{equation}

Specifically, SpecCascade~\citep{narasimhan2024faster} and Medusa~\citep{cai2024medusa} are examples of truncation-based verification, with allowed sets defined by min-p (\Cref{eq:min-p}) and $\eta$-sampling (\Cref{eq:eta}), respectively. Their reported gains over standard speculative decoding conflate the effect of verification with the effect of truncation sampling itself. When evaluated against the appropriately truncated target, the actual performance gap becomes clear, as shown in \Cref{sec:exp-trunc}.

Having characterized existing lossy verification into collaborative and truncation-based verification, we now evaluate both paradigms empirically. As foreshadowed in the introduction, benchmark difficulty is itself diagnostic: easy tasks such as GSM8K mask the distortions introduced by lossy verification
We therefore evaluate on four harder benchmarks spanning distinct domains: MATH~\cite{hendrycks2021measuring} for mathematical reasoning, MBPP+~\cite{liu2023your} for code generation, INCLUDE~\citep{romanou2024include} for multilingual understanding, and BFCL~\citep{berkeley-function-calling-leaderboard} for agentic tool use. For efficiency we report Block Efficiency (BE), the number of tokens accepted per step, and Decoding Speed (DS), tokens per second; for task quality we report Accuracy on MATH, INCLUDE, and BFCL and Pass@1 on MBPP+. Configurations are given in \Cref{sec:exp_setup}.



\begin{figure*}[t!]
    \centering
    \includegraphics[width=0.84\textwidth]
        {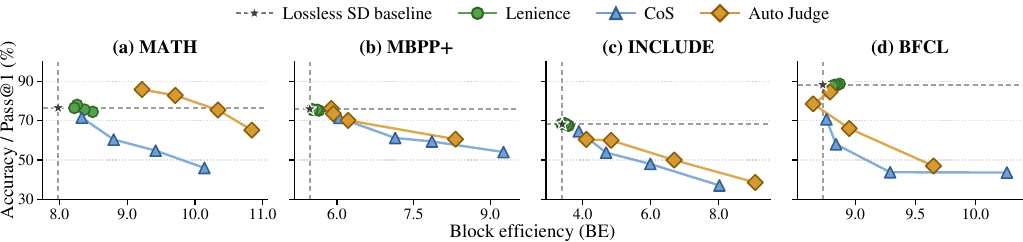}
    \caption{Efficiency and task performance tradeoff for collaborative verification
    across four benchmarks (Qwen2.5-72B/0.5B-Instruct-GPTQ-Int8). Each marker represents one hyperparameter
    setting, and lines connect settings in sweep order. Full results are in \Cref{sec: total_results}; the Llama pair and extended analysis are in \Cref{app:trade_off_extended}.
    }
    \label{fig:trade col}
\end{figure*}

\section{Identifying the Key Factor in Collaborative Verification}
\label{sec:exp-collaborative}

As analyzed in \Cref{sec: col vrf}, Lenience-based, interpolation-based, and
judge-based relaxation all fall under the broader category of collaborative
verification. We use Lenience~\cite{leviathan2023fast},
CoS~\cite{fu2025fastlargelanguagemodel}, and
AutoJudge~\cite{garipov2025autojudge} as representative methods of these
three subclasses, respectively.

\subsection{Overshoot Suppression Improves Performance}
\label{sec:exp-overshoot}

Both CoS and AutoJudge  exhibit a clear efficiency-performance tradeoff once relaxation becomes aggressive (see~\Cref{fig:trade col}). As illustrated by \Cref{eq:we_yield,eq:judge_pgenerate}, the generated distribution increasingly leans toward the draft distribution, and the relaxed acceptance rule leads to degraded task performance. In contrast, we observe that lenience-based relaxation maintains strong task performance while achieving slightly better efficiency than the baseline (\Cref{fig:trade col}). Since $q(x)$ is effectively the worst-case of overestimation region ($q>p$), the gain from lenience must come from either the adaptive interpolation in the underestimated region or the thresholding in the overshooting region. 

To disentangle these contributions, we further conduct an ablation study on MBPP+ to investigate the key factor behind the effectiveness of lenience-based relaxation. We selectively replace the generated distribution of CoS in the $q(x) < p(x)$ and $q(x) > p(x)/\lambda$ regions with either ceiling the overshoot region or the adaptive interpolation in the underestimated region, and then measure the resulting changes in efficiency and task performance tradeoff. As shown in \Cref{tab:ablation_yield_rule}, introducing adaptive interpolation in the undershoot region has similar severe tradeoff as CoS. In contrast, ceiling the overshooting region alone achieves task performance comparable to lossless verification, suggesting that ceiling is its primary contribution.




In summary, the ablation identifies the overshoot region, not the underestimated region, as the source of lenience-based relaxation's effectiveness: with the overshoot ceiling alone, task performance matches lossless verification while efficiency still improves, whereas the adaptive interpolation in the underestimated region inherits CoS's severe trade-off. This aligns with recent findings that low-quality generations arise primarily from overshoot tokens~\cite{zhou2025balancing, yue2025does, fan2026flatter}, and points to overshoot suppression, rather than interpolation, as the key design principle for effective collaborative verification.

\begin{table}[ht]
\centering
\small
\setlength{\tabcolsep}{4pt}
\renewcommand{\arraystretch}{1.2}
\caption{Ablation of the two mechanisms in lenience-based relaxation on
MBPP+.}
\label{tab:ablation_yield_rule}
\begin{tabular}{@{}lS[table-format=1.2]S[table-format=2.2]@{\hspace{14pt}}S[table-format=1.2]S[table-format=2.2]@{}}
\toprule

\addlinespace[4pt]
$\boldsymbol{\lambda}$ & {BE} & {Pass@1} & {BE} & {Pass@1} \\
\midrule
& \multicolumn{2}{c}{\makebox[0.40\columnwidth][c]{\hspace*{-10pt}\scriptsize\color{black!60}
$\begin{cases} q + \Delta_\ell (p-q), & q \le p\\ \lambda p + (1-\lambda) q, & q > p \end{cases}$}}
& \multicolumn{2}{c}{\makebox[0.40\columnwidth][c]{\hspace*{-10pt}\scriptsize\color{black!60}
$\begin{cases} \lambda p + (1-\lambda) q, & q \le p/\lambda\\ p/\lambda, & q > p/\lambda \end{cases}$}} \\
0.2 & 9.15 & 50.26 & 5.59 & 75.93 \\
0.4 & 7.93 & 56.61 & 5.57 & 75.13 \\
0.6 & 6.80 & 60.85 & 5.54 & 75.66 \\
0.8 & 5.95 & 66.14 & 5.54 & 75.13 \\
\bottomrule
\end{tabular}
\end{table}

\subsection{A Well-Trained Judge Is Promising} 
\label{sec:exp-judge-domain}

AutoJudge achieves the best trade-off among the collaborative methods. It
outperforms CoS over most of the sweep on every benchmark, and on MATH it
beats even the lossless baseline across a wide BE range for both model pairs~\Cref{fig:trade col,fig:trade_llama}.
Its weakness is stability: the sweep is non-monotone on some benchmarks, and
the margin over CoS varies by domain. We
attribute this pattern to the judge's supervision rather than to its resulting distribution~\Cref{eq:judge_pgenerate}. We train a single judge on GSM8K and LiveCodeBench mismatches according to the official release
(\Cref{app:autojudge-implementation}). GSM8K contributes 7{,}473 problems against LiveCodeBench's 880, so the training data is mostly mathematical reasoning, and
MATH is where the judge beats CoS by the widest margin.

To test this directly, we sweep the training data mixture from math-dominated ($9{:}1$) to code-dominated ($1{:}9$). As shown in~\Cref{tab:mix_sweep_q}, interestingly, the math-dominated judge outperforms the code-dominated one on every benchmark, including MBPP+: upweighting LiveCodeBench does not improve the one benchmark closest to it. The math-heavy judges also stay usable across all four benchmarks, whereas the code-heavy ones trade task performance for acceptance.

\begin{table}[ht]
  \centering
  \small
  \setlength{\tabcolsep}{3.5pt}
  \caption{Mixture-ratio sweep. Each cell averages the
  four acceptance thresholds $\delta\in\{0.2, 0.4, 0.6, 0.8\}$.}
  \label{tab:mix_sweep_q}
  \begin{tabular}{@{}l cc cc cc cc@{}}
    \toprule
    \multirow{2}{*}{\textbf{Mix}}
      & \multicolumn{2}{c}{\textbf{MATH}}
      & \multicolumn{2}{c}{\textbf{MBPP+}}
      & \multicolumn{2}{c}{\textbf{INCLUDE}}
      & \multicolumn{2}{c}{\textbf{BFCL}} \\
    \cmidrule(lr){2-3}\cmidrule(lr){4-5}\cmidrule(lr){6-7}\cmidrule(l){8-9}
      & BE & Acc & BE & P@1 & BE & Acc & BE & Acc \\
    \midrule
    9:1 & 8.62 & \textbf{83.10} & 5.82 & \textbf{76.72} & 3.90 & \textbf{62.96} & 8.89 & \textbf{88.00} \\
    7:3 & 9.67 & 73.95 & 6.65 & 69.84 & 6.50 & 59.09 & 9.39 & 77.38 \\
    5:5 & 9.67 & 74.55 & 7.60 & 60.12 & 7.12 & 47.50 & 9.53 & 77.00 \\
    3:7 & 9.71 & 73.40 & 6.66 & 71.63 & 6.98 & 50.91 & 9.49 & 82.50 \\
    1:9 & 10.55 & 64.45 & 8.34 & 57.87 & 8.89 & 43.52 & 9.92 & 63.00 \\
    \bottomrule
  \end{tabular}
\end{table}


In conclusion, the mixture sweep isolates supervision quality as what governs judge behavior, which makes judge-based relaxation promising. Unlike interpolation-based and lenience-based relaxation, whose induced distributions are fixed in closed form, the judge's ceiling is set by its supervision and moves as the supervision improves. With sufficient supervision it already beats lossless SD on both efficiency and task performance, so improving judge supervision is a concrete direction for extending that result beyond a single domain.

\section{Revealing the Pitfall in Truncation-Based Verification}
\label{sec:exp-trunc}

We re-evaluate two representative truncation-based verification methods, i.e., typical acceptance~\citep{cai2024medusa} and SpecCascade~\citep{narasimhan2024faster}, against their matched truncation sampling baselines ($\eta$-sampling and Min-$p$ sampling, respectively) directly applied on the target model with lossless verification~\citep{leviathan2023fast}. This setup isolates whether the reported gains stem from the  verification mechanism itself or from the underlying truncation sampling.  Single-draft experiments are conducted within the standard SD framework~\cite{leviathan2023fast}, and multi-draft experiments within EAGLE-3~\cite{li2025eagle}.

\subsection{Standard SD}


Across multiple benchmarks (\Cref{fig:combined_trend}), truncation-based verification exhibits a clear performance gap relative to its matched baseline under standard speculative decoding. Min-p sampling and $\eta$-sampling outperform SpecCascade and typical acceptance respectively.
As formalized in \Cref{def:truncation-verification}, this degradation arises because tokens within the allowed set are always accepted, producing a distribution that mirrors the draft model rather than the target distribution (see \Cref{fig:verification_mechanisms}). Overall, these results reinforce our claim that the apparent gains of truncation-based verification are largely driven by truncation itself, while the verification mechanism can further distort the decoding distribution and reduce task performance.

Beyond the performance gaps, we investigate truncation sampling effect in Lemma~\ref{lem:trunc-be}. It decomposes the truncation-induced change in acceptance probability into a non-negative \emph{gain} term, accumulated over the retained support $A$, and a \emph{loss} term, accumulated over the discarded region at the tail (proof see \Cref{app:be-tree}). $\Delta\mathrm{BE}$ thus determines whether truncation sampling hurts SD efficiency.

\begin{lemma}[Truncation sampling efficiency effect]
\label{lem:trunc-be}
Denote allowed set $\mathcal{A}_\Theta \subseteq \mathcal{V}$ determined by a truncation strategy $\Theta$, $z = \textstyle\sum_{v \in \mathcal{A}_\Theta} p(v)$.
The change in single-token acceptance probability induced by truncation sampling satisfies
\begin{equation}
\begin{aligned}
\Delta\mathrm{BE}
= \sum_{x \in \mathcal{A}_\Theta}
  \min\big\{ (q(x) - p(x))_+,\, \\(\tfrac{1}{z} - 1) p(x) \big\} 
- \sum_{x \notin \mathcal{A}_\Theta} \min\big(p(x), q(x)\big).
\end{aligned}
\label{eq:delta-be}
\end{equation}
\end{lemma}

To make this concrete, \Cref{fig:net_comparison} plots the two terms from Lemma~\ref{lem:trunc-be} across a range of distribution pairs. Under min-$p$ sampling, the gain dominates at smaller $p_{\text{base}}$, but the margin narrows  as $p_{\text{base}}$ approaches $0.9$, where the discarded mass begins to outweigh the redistributed gain. Under $\eta$-sampling, $\Delta\mathrm{BE} > 0$ uniformly across all distribution pairs and $\epsilon$, and the net gain grows monotonically with $\epsilon$ despite wider confidence bands. This indicates that $\eta$-sampling reliably converts mass from the truncated tail into additional acceptance probability on the retained support, whereas min-$p$ does not.

\begin{figure}[ht]
    \centering
    \includegraphics[width=\columnwidth]
        {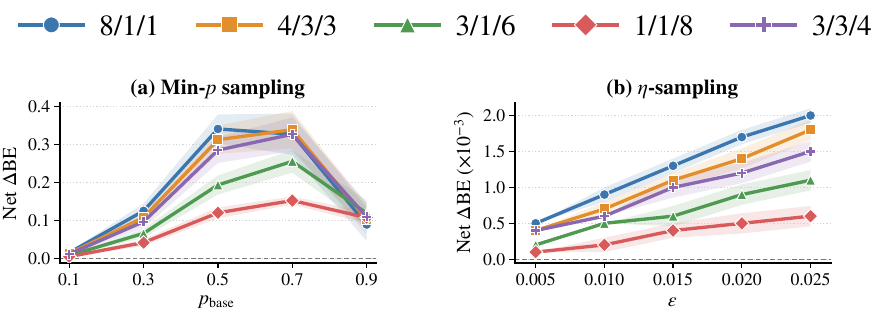}
    \caption{
    Net change in $\Delta\mathrm{BE}$ under different
    draft--target alignment ratios for Min-$p$ and $\eta$-sampling.
    Each triplet reports the ratios of matching, partially overlapping,
    and unrelated candidate tokens, respectively.
    }
    \label{fig:net_comparison}
\end{figure}

\begin{figure*}[t]
    \centering
    \includegraphics[width=0.9\textwidth]
        {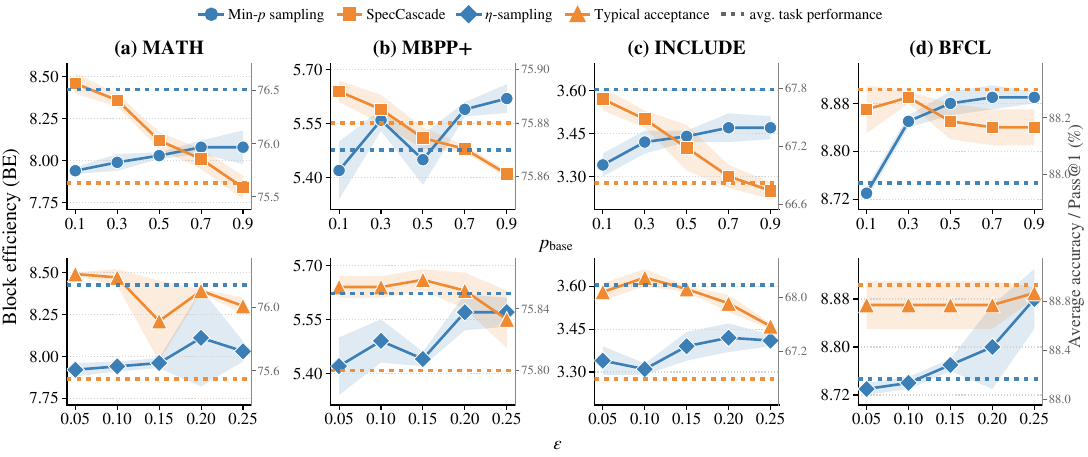}
    \caption{
    Block efficiency and task performance of truncation-based methods
    using the Qwen2.5-72B and Qwen2.5-0.5B pair across four benchmarks.
    \textit{Top row}: Min-$p$ sampling and SpecCascade.
    \textit{Bottom row}: $\eta$-sampling and typical acceptance.
    Solid lines report mean block efficiency (BE).
    Colored dotted horizontal lines report accuracy or Pass@1, averaged over hyperparameter settings. 
    }
    \label{fig:combined_trend}
\end{figure*}

\begin{table*}[t]
\centering
\scriptsize
\caption{Evaluation results using LLaMA-3.1 8B and its official released pair under EAGLE-3, averaged across the parameter grid for each method. \textbf{Bold} denotes the better task performance within each matched pair. \textcolor{gray}{Gray} denotes the EAGLE-3 baseline, which is not directly comparable. 
}
\label{tab:eagle_multi_avg}
\setlength{\tabcolsep}{4pt}
\renewcommand{\arraystretch}{1.15}
\resizebox{1\linewidth}{!}{
\begin{tabular}{lccc ccc ccc ccc}
\toprule
\multirow{2}{*}{\textbf{Method}}
& \multicolumn{3}{c}{\textbf{MATH}}
& \multicolumn{3}{c}{\textbf{MBPP+}}
& \multicolumn{3}{c}{\textbf{INCLUDE}}
& \multicolumn{3}{c}{\textbf{BFCL}} \\
\cmidrule(lr){2-4}\cmidrule(lr){5-7}\cmidrule(lr){8-10}\cmidrule(lr){11-13}
& BE & DS & Acc\,(\%) & BE & DS & Pass@1\,(\%) & BE & DS & Acc\,(\%) & BE & DS & Acc\,(\%) \\
\midrule
\textcolor{gray}{EAGLE-3}~\citep{li2025eagle}
 & \textcolor{gray}{3.76} & \textcolor{gray}{140.21} & \textcolor{gray}{73.20} & \textcolor{gray}{4.70} & \textcolor{gray}{167.83} & \textcolor{gray}{59.30} & \textcolor{gray}{0.67} & \textcolor{gray}{45.72} & \textcolor{gray}{35.50} & \textcolor{gray}{2.57} & \textcolor{gray}{85.27} & \textcolor{gray}{86.00} \\
\midrule
Min-$p$ sampling~\citep{nguyen2025turningheatminpsampling}
 & 3.98 & 143.22 & \textbf{76.88} & 4.90 & 168.57 & \textbf{61.80} & 0.56 & 45.80 & \textbf{37.00} & 2.58 & 84.26 & \textbf{86.60} \\
SpecCascade~\citep{narasimhan2024faster}
 & 4.22 & 152.74 & 76.16 & 5.00 & 175.91 & 60.74 & 0.78 & 48.13 & 33.27 & 2.59 & 85.20 & 85.40 \\
\midrule
$\eta$-sampling~\citep{hewitt2022truncation}
 & 3.93 & 140.05 & \textbf{76.12} & 4.87 & 165.61 & \textbf{62.17} & 0.55 & 45.33 & \textbf{36.73} & 2.58 & 79.77 & \textbf{86.30} \\
Typical acceptance~\citep{cai2024medusa}
 & 4.61 & 162.73 & 69.84 & 5.20 & 180.10 & 56.88 & 1.08 & 55.54 & 27.91 & 2.64 & 85.61 & 81.40 \\
\bottomrule
\end{tabular}}
\end{table*}



The contrasting trends in \Cref{fig:combined_trend} further show that the benchmark experiments are consistent with both the theoretical analysis in Lemma~\ref{lem:trunc-be} and the simulation in \Cref{fig:net_comparison}. The lemma predicts that the efficiency effect of truncation depends on the balance between the gain and the loss. This balance is reflected in the simulation: min-$p$ sampling yields positive gains at moderate thresholds but becomes less favorable as $p_{\text{base}}$ approaches $0.9$, whereas $\eta$-sampling maintains positive $\Delta\mathrm{BE}$ across the tested range. The benchmark results follow the same pattern: min-$p$ sampling rises to a peak and then declines, while $\eta$-sampling increases more consistently, with only minor deviations. Overall, these aligned trends support the theoretical decomposition and the controlled simulation.

\subsection{EAGLE-3}

Beyond standard SD, EAGLE-3 is of particular interest because it operates over multiple drafts arranged in a tree. We are therefore intrigued to investigate whether lossy verification behaves differently in this setting. \Cref{thm:kl-gap} answers this by characterizing the divergence from the matched truncation sampling target under both schemes.

\begin{theorem}[Draft-target divergence under tree verification]
\label{thm:kl-gap}
For a strategy $\Theta$ with allowed set $\mathcal{A}_\Theta$, let
$Z_\Theta(p) = \sum_{x \in \mathcal{A}_\Theta} p(x)$ and
$Z_\Theta(q) = \sum_{x \in \mathcal{A}_\Theta} q(x)$ denote the retained
target and draft mass. The KL divergence between the distributions generated
by truncation-based verification and truncation sampling is
\begin{equation}
\label{eq:kl_gap}
\resizebox{\columnwidth}{!}{$
\mathrm{KL}_{\mathrm{SD}} = \mathbb{E}\bigl[ \sum_{x\in\mathcal{A}_\Theta}
\tfrac{q(x)}{Z_\Theta(q)} \log \tfrac{q(x) Z_\Theta(p)}{Z_\Theta(q) p(x)} \bigr],
\quad
\mathrm{KL}_{\mathrm{EAGLE}} = \mathbb{E}\bigl[ \log \tfrac{Z_\Theta(p)}{p(x)} \bigr]
$}\notag
\end{equation}
for standard SD and EAGLE-3 respectively, and these behave differently as the
draft improves:
$\lim_{q\to p}\mathrm{KL}_{\mathrm{SD}} = 0$, whereas
$\mathrm{KL}_{\mathrm{EAGLE}} > 0$.
\end{theorem}


In particular, $\mathrm{KL}_{\mathrm{SD}} < \mathrm{KL}_{\mathrm{EAGLE}}$ for all $q$ sufficiently close to $p$: improving draft drives per-token distortion to zero under standard SD but not under EAGLE-3. We therefore expect task-performance gap between truncation-based verification and truncation sampling larger under EAGLE-3.


Consistent with the theoretical prediction, under EAGLE-3, where the divergence in \Cref{thm:kl-gap} does not vanish, the average performance gap to the matched truncation sampling baseline widens for both SpecCascade and typical acceptance (\Cref{tab:eagle_multi_avg}), compared with standard SD (\Cref{fig:combined_trend}). The degradation is severe enough that typical acceptance falls below the EAGLE-3 baseline on all four benchmarks, and SpecCascade on two, while the matched truncation sampling baselines stay above it. What the truncation-based rules buy in return is a small block-efficiency gain (\Cref{app:be-tree}), which does not offset the distortion.

\section{Conclusion}
\label{conclusion}

We provide a unified view of lossy verification in speculative decoding, showing that existing methods fall into two paradigms: truncation-based verification and collaborative verification. Our analysis reveals a critical pitfall of truncation-based methods, where performance can collapse relative to the true baseline due to distributional distortion, especially in EAGLE-3, and identifies the importance of controlling overshoot in lossy verification to preserve generation quality. These findings offer principled guidance for speculative decoding that enable faster, quality-preserving LLM inference.


\section*{Limitations}
Our study has several limitations. First, the evaluation focuses on reasoning, code generation, multilingual understanding, and function calling, as represented by MATH, MBPP+, INCLUDE, and BFCL. All four admit automatic verification against a reference answer, which makes the quality side of the speed–quality trade-off measurable, but it also means our evaluation excludes open-ended generation and dialogue. This exclusion is a scope condition on our conclusions rather than a matter of coverage, because truncation sampling was introduced for open-ended generation: min-p sampling~\cite{nguyen2025turningheatminpsampling} and $\eta$-sampling~\cite{hewitt2022truncationsamplinglanguagemodel} are designed to suppress incoherent continuations from the unreliable tail, and that benefit is largely invisible on tasks scored by a single extractable answer. Second, our unification centers on verification mechanisms rather than the complete speculative decoding pipeline; therefore, it does not fully capture interactions with draft-model design, candidate construction, tree expansion, or system-level scheduling. Finally, wall-clock speedups depend on the underlying hardware, implementation, and serving stack. Although we report block efficiency as a comparatively hardware-agnostic proxy, the exact speed–quality trade-off may vary across deployment environments.

\section*{Ethics Statement}

This work analyzes verification schemes for speculative decoding and does not involve human participants, private or sensitive data, or the collection of new datasets. All experiments use publicly available models and established benchmarks in accordance with their respective licenses and intended research purposes.

Our findings aim to improve the reliability of accelerated LLM inference by identifying quality-degradation and distributional-distortion failure modes that may otherwise be overlooked. Although this work does not introduce new model capabilities, improvements in inference efficiency could indirectly lower the cost of deploying existing language models, including for potentially harmful applications. Moreover, lossy verification may affect generation quality differently across tasks, languages, or domains not covered by our evaluation. We therefore recommend evaluating both efficiency and output quality before deployment, particularly in high-stakes settings.

We report the models, benchmarks, evaluation settings, and hardware configurations used in our experiments to support transparency and reproducibility.

\bibliography{Main}

\clearpage
\onecolumn
\appendix

\section{Detailed Derivations for Collaborative Verification Methods}
\label{sec:detailed_derivation}

This section derives the generation distributions induced by Lenience, Judge-Based Relaxation, DIVERSED and
Fuzzy Speculative Decoding (Fuzzy SD). For DIVERSED and Fuzzy
SD, the resulting distribution can be interpreted as a context-dependent
combination of the target distribution $p$ and draft distribution $q$.

\subsection{DIVERSED}
\label{sec:diversed_derivation}

DIVERSED~\citep{wang2026diversed} generalizes the static ensemble used
by CoS by predicting a task- and context-dependent ensemble weight.
At decoding step $t$, the method computes
\begin{equation}
\lambda_{t}
=
f\!\left(h_t^q,h_t^p\right),
\qquad
\lambda_{t}\in[0,1],
\label{eq:diversed_weight}
\end{equation}
where $h_t^q$ and $h_t^p$ are representations produced by the draft and
target models. For brevity, we omit the time index and write
$\lambda=\lambda_{t}$.

DIVERSED constructs the context-dependent ensemble distribution
\begin{equation}
\nu(x)
=
\lambda p(x)
+
\bigl(1-\lambda\bigr)q(x).
\label{eq:diversed_ensemble}
\end{equation}

The draft token is then verified against $\nu$, rather than
directly against $p$, using
\begin{equation}
h(x)
=
\min\!\left(
1,\frac{\nu(x)}{q(x)}
\right).
\label{eq:diversed_acceptance}
\end{equation}

Upon rejection, the corresponding speculative-sampling correction
distribution is
\begin{equation}
r(x)
=
\frac{
\bigl(\nu(x)-q(x)\bigr)_+
}{
\sum_{v\in\mathcal V}
\bigl(\nu(v)-q(v)\bigr)_+
}.
\label{eq:diversed_residual}
\end{equation}

Let
\begin{equation}
R
=
1-\sum_{v\in\mathcal V}q(v)h(v)
\end{equation}
denote the rejection probability. Since
\begin{equation}
q(x)h(x)
=
\min\!\left(q(x),\nu(x)\right),
\end{equation}
we have
\begin{equation}
R
=
\sum_{v\in\mathcal V}
\bigl(q(v)-\nu(v)\bigr)_+.
\label{eq:diversed_rejection}
\end{equation}

Because both $q$ and $\nu$ are normalized distributions,
\begin{equation}
\sum_{v\in\mathcal V}
\bigl(q(v)-\nu(v)\bigr)_+
=
\sum_{v\in\mathcal V}
\bigl(\nu(v)-q(v)\bigr)_+.
\label{eq:diversed_mass_equality}
\end{equation}

Substituting
\Cref{eq:diversed_acceptance,eq:diversed_residual} gives
\begin{align}
P(\mathrm{generate}\,x)
&=
\min\!\left(q(x),\nu(x)\right)
+
R r(x)\\
&=
\min\!\left(q(x),\nu(x)\right)
+
\bigl(\nu(x)-q(x)\bigr)_+\\
&=
\nu(x).
\end{align}

Therefore,
\begin{equation}
\boxed{
P(\mathrm{generate}\,x)
=
\lambda p(x)
+
\bigl(1-\lambda\bigr)q(x).
}
\label{eq:diversed_yield}
\end{equation}

DIVERSED thus belongs directly to collaborative verification. Unlike
CoS, which uses a fixed global coefficient, DIVERSED predicts
$\lambda$ from the current task and decoding context. However, at
each individual decoding position, the coefficient remains shared
across all vocabulary tokens.

\subsection{Fuzzy Speculative Decoding}
\label{sec:fuzzy_derivation}

Fuzzy SD~\citep{holsman2025fuzzy} makes a deterministic acceptance
decision based on the divergence between the full draft and target
distributions at the current position. Define the context-level gate
\begin{equation}
g_T
=
\mathbb I
\left[
D\!\left(p(\cdot),q(\cdot)\right)<T
\right],
\qquad
g_T\in\{0,1\},
\label{eq:fuzzy_gate}
\end{equation}
where $D$ is a distributional divergence and $T$ is a user-specified
threshold.

Because $g_T$ depends on the complete pair of distributions rather than
on the sampled token, the acceptance probability is
\begin{equation}
h_T(x)=g_T
\qquad
\text{for every }x\in\mathcal V.
\label{eq:fuzzy_acceptance}
\end{equation}

When $g_T=1$, the divergence is below the threshold and every candidate
drawn from $q$ is accepted. Therefore,
\begin{equation}
P(\mathrm{generate}\,x\mid g_T=1)=q(x).
\label{eq:fuzzy_accept_case}
\end{equation}

When $g_T=0$, the candidate is rejected and Fuzzy SD uses the target
decoding distribution for correction. Thus,
\begin{equation}
P(\mathrm{generate}\,x\mid g_T=0)=p(x).
\label{eq:fuzzy_reject_case}
\end{equation}

Equivalently, substituting $h_T(x)=g_T$ and $r(x)=p(x)$ yields
\begin{align}
P(\mathrm{generate}\,x)
&=
g_Tq(x)
+
\left(
1-\sum_{v\in\mathcal V}g_Tq(v)
\right)p(x)\\
&=
g_Tq(x)
+
(1-g_T)p(x),
\end{align}
because $\sum_vq(v)=1$. Hence,
\begin{equation}
\boxed{
P(\mathrm{generate}\,x)
=
(1-g_T)p(x)+g_Tq(x).
}
\label{eq:fuzzy_yield}
\end{equation}

Fuzzy SD can therefore be interpreted as a binary,
context-dependent collaborative verification method:
\begin{equation}
P(\mathrm{generate}\,x)
=
\begin{cases}
q(x),
&
D\!\left(p(\cdot),q(\cdot)\right)<T,\\[3pt]
p(x),
&
D\!\left(p(\cdot),q(\cdot)\right)\ge T.
\end{cases}
\label{eq:fuzzy_piecewise}
\end{equation}

Unlike CoS and DIVERSED, Fuzzy SD does not use a soft interpolation
coefficient in $(0,1)$. Instead, it selects one endpoint of the
$p$--$q$ interpolation according to the current distributional
discrepancy.

\subsection{Lenience}
\label{sec:lenience_derivation}
In tokenwise speculative sampling~\citep{leviathan2023fast}, each token $x$ is drafted from $q(x)$ and verified against $p(x)$. It is accepted with probability $h(x)=\min\{1,\,p(x)/\ell q(x)\}$, or rejected and replaced from $P_{\text{res}}(x)$. Thus the yield probability is:

\begin{align}
&P(\text{generate} \, x) \nonumber=  P(\text{draft and accept} \,x)\nonumber+ P(\text{draft and reject} \,v,\, \text{resampled} \, x)\nonumber \\
&=q(x)h(x)+\Bigl[\sum_{v\in \mathcal{V}} q(v)(1-h(v))\Bigr]\nonumber\times\frac{p(x) - \min\{p(x),q(x)\}}{\textstyle\sum_{v\in \mathcal{V}} \bigl(p(v) - \min\{p(v),q(v)\}\bigr)}\nonumber \\
&=\min\Bigl\{q(x),\,\frac{p(x)}{\ell }\Bigr\}\nonumber+\frac{\textstyle\sum_{v\in \mathcal{V}} \bigl(q(v)-\min\{q(v),\,p(v)/\ell \}\bigr)}{\textstyle\sum_{v\in \mathcal{V}} \bigl(q(v) - \min\{q(v),\,p(v)\}\bigr)}\nonumber\times\Bigl[p(x) - \min\{p(x),q(x)\}\Bigr]\nonumber\\
&=\min\Bigl\{q(x),\,\frac{p(x)}{\ell }\Bigr\}\nonumber+ \Delta_\ell\Bigl[p(x) - \min\{p(x),q(x)\}\Bigr]\nonumber\\
&=\begin{cases}
q(x)+\Delta_\ell(p(x)-q(x)),  &q(x)\le p(x),\\
q(x), & p(x)\le q(x)\le p(x)/\ell,\\
p(x)/\ell, &q(x)\ge p(x)/\ell.
\end{cases} \nonumber
\end{align}

\paragraph{Derivation for $\Delta_\ell$}

\begin{align}
\Delta_\ell
&=
\frac{
\sum_{v\in \mathcal{V}}
\left(
q(v)-\min\left\{q(v),\,\frac{p(v)}{\ell}\right\}
\right)
}{
\sum_{v\in \mathcal{V}}
\left(
q(v)-\min\left\{q(v),\,p(v)\right\}
\right)
}.
\end{align}

Using the identity
\begin{align}
a-\min\{a,b\}
=
\frac{|a-b|+(a-b)}{2},
\end{align}
the numerator becomes
\begin{align}
&\sum_{v\in \mathcal{V}}
\left(
q(v)-\min\left\{q(v),\,\frac{p(v)}{\ell}\right\}
\right)=
\frac12
\sum_{v\in \mathcal{V}}
\left|
q(v)-\frac{p(v)}{\ell}
\right|+
\frac12
\sum_{v\in \mathcal{V}}
\left(
q(v)-\frac{p(v)}{\ell}
\right).
\end{align}
Since \(p\) and \(q\) are probability distributions on \(\mathcal{V}\),
\begin{align}
\sum_{v\in \mathcal{V}} q(v)=1,
\qquad
\sum_{v\in \mathcal{V}} p(v)=1,
\end{align}
we have
\begin{align}
\sum_{v\in \mathcal{V}}
\left(
q(v)-\frac{p(v)}{\ell}
\right)
=
1-\frac{1}{\ell}.
\end{align}
Therefore,
\begin{align}
&\sum_{v\in \mathcal{V}}
\left(
q(v)-\min\left\{q(v),\,\frac{p(v)}{\ell}\right\}
\right)
=
\frac12
\sum_{v\in \mathcal{V}}
\left|
q(v)-\frac{p(v)}{\ell}
\right|
+
\frac12-\frac{1}{2\ell}.
\end{align}

Similarly, the denominator satisfies
\begin{align}
\sum_{v\in \mathcal{V}}
\left(
q(v)-\min\left\{q(v),\,p(v)\right\}
\right)
=
\frac12
\sum_{v\in \mathcal{V}}
\left|
q(v)-p(v)
\right|
+
\frac12
\sum_{v\in \mathcal{V}}
\left(
q(v)-p(v)
\right).
\end{align}

Again using
\begin{align}
\sum_{v\in \mathcal{V}} q(v)
=
\sum_{v\in \mathcal{V}} p(v)
=
1,
\end{align}
we obtain
\begin{align}
\sum_{v\in \mathcal{V}}
\left(
q(v)-p(v)
\right)
=
0,
\end{align}
and hence
\begin{align}
\sum_{v\in \mathcal{V}}
\left(
q(v)-\min\left\{q(v),\,p(v)\right\}
\right)
=
\frac12
\sum_{v\in \mathcal{V}}
\left|
q(v)-p(v)
\right|.
\end{align}

Substituting these two expressions back into \(\Delta_\ell\) yields
\begin{align}
\Delta_\ell
&=
\frac{
\frac12
\sum_{v\in \mathcal{V}}
\left|
q(v)-\frac{p(v)}{\ell}
\right|
+
\frac12-\frac{1}{2\ell}
}{
\frac12
\sum_{v\in \mathcal{V}}
\left|
q(v)-p(v)
\right|
}.
\end{align}
Thus,
\begin{align}
\boxed{P(\text{generate} \, x)=
\begin{cases}
\Delta_\ell p(x)+(1-\Delta_\ell)q(x),
& q(x)\le p(x),\\[3pt]
q(x),
& p(x)<q(x)\le p(x)/\ell,\\[3pt]
p(x)/\ell,
& q(x)\ge p(x)/\ell.
\end{cases}
}
\end{align}

\section{Lenience-based Relaxation Interpretation}
\label{sec: further_interpretation}
\begin{figure*}[t]
    \centering
    \setlength{\abovecaptionskip}{4pt}
    \setlength{\belowcaptionskip}{0pt}
    \begin{subfigure}{0.45\textwidth}
        \includegraphics[width=\linewidth]{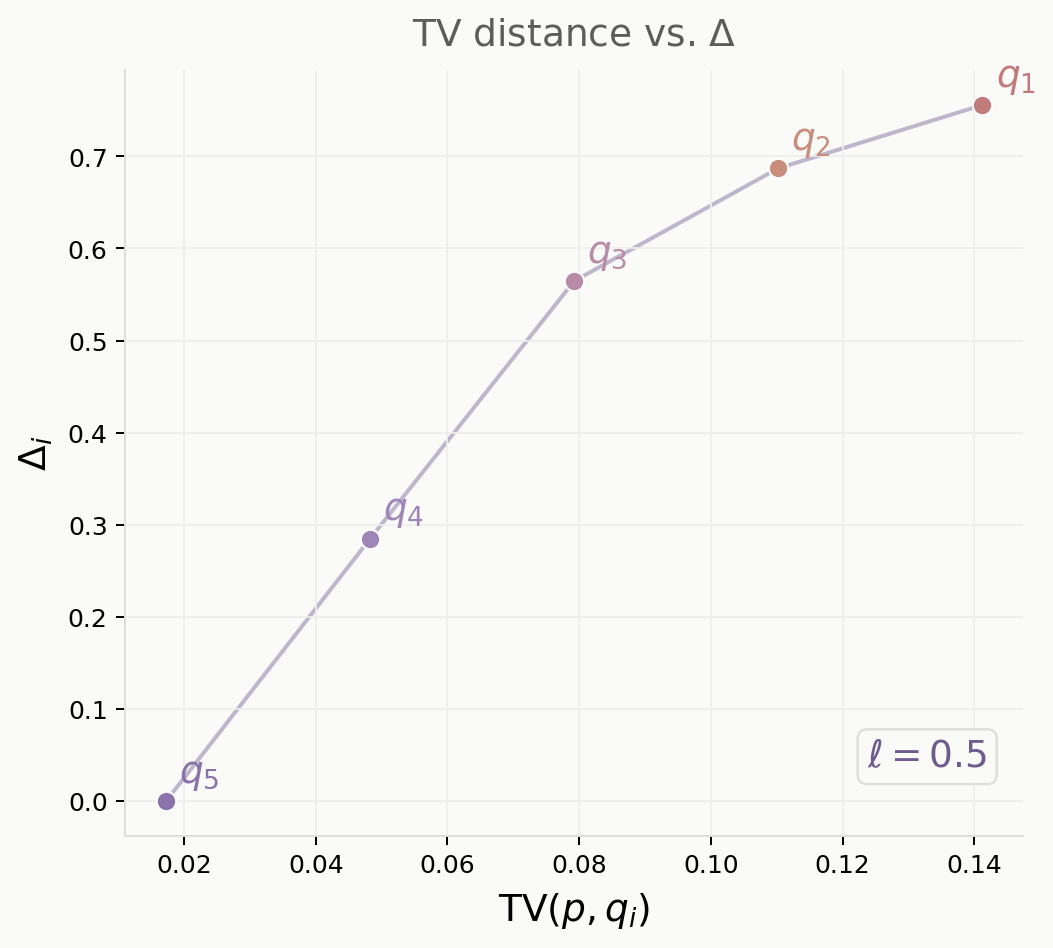}
        \caption{Adaptive interpolation.}
        \label{fig:Adaptive}
    \end{subfigure}
    \hfill
    \begin{subfigure}{0.45\textwidth}
        \includegraphics[width=\linewidth]{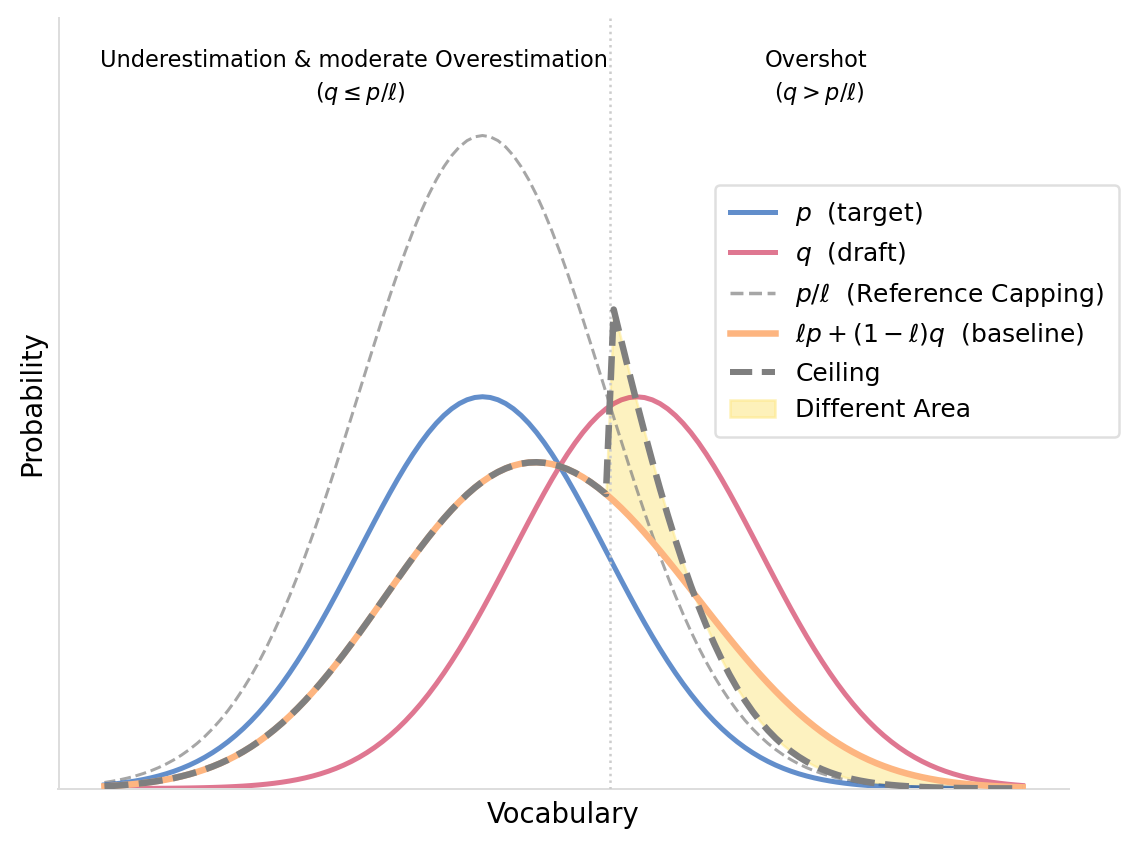}
        \caption{Overshoot ceiling.}
        \label{fig:ceiling}
    \end{subfigure}
    \caption{Two views of lenience-based relaxation.
    Fig.\subref{fig:Adaptive} illustrates how the interpolation strength
    adapts to the draft--target divergence on a per-token basis.
    Fig.\subref{fig:ceiling} contrasts a global linear mixture with a
    pointwise ceiling at $p/\ell$, highlighting the asymmetric treatment
    of overshoot.}
    \label{fig:lenience_two}
\end{figure*}

Lenience-based relaxation trades a controlled amount of distributional
distortion for a higher acceptance rate, governed by the lenience
parameter~$\ell$. ~\Cref{fig:lenience_two} examines two facets of
this trade-off: \emph{how much} relaxation to apply, and \emph{where} it
takes effect.

\paragraph{Adaptive interpolation.}
\Cref{fig:lenience_two}\subref{fig:Adaptive} plots the interpolation strength $\Delta_i$
against the draft--target divergence $\mathrm{TV}(p, q_i)$. Rather than
applying a fixed amount of leniency, the relaxation adapts to local
disagreement: when the draft already agrees with the target
($q_5$, small $\mathrm{TV}$) we have $\Delta_i \approx 0$ and almost no
relaxation is applied, whereas a more divergent draft ($q_1$) induces a
larger $\Delta_i$. The concave profile further shows that the additional
leniency saturates rather than growing without bound as divergence
increases, so relaxation is spent precisely where draft and target
diverge and withheld where they coincide.

\paragraph{Overshoot ceiling.}
\Cref{fig:lenience_two}\subref{fig:ceiling} contrasts a global linear mixture
$\ell p + (1-\ell)q$, which blends draft and target uniformly across the
vocabulary, with a pointwise ceiling at $p/\ell$ (reference capping),
which accepts the draft wherever it stays below the ceiling and clips it
only where it exceeds it. Two regimes emerge: for $q \le p/\ell$ the
draft lies under the ceiling (underestimation or moderate overestimation)
and is accepted, while for $q > p/\ell$ the draft \emph{overshoots} the
relaxed target and is capped. The shaded region marks exactly where the
two schemes disagree---the overshoot mass that the ceiling clips but the
linear baseline retains---highlighting the asymmetric treatment of
overshoot.

\section{Experimental Setup}
\label{sec:exp_setup}

\paragraph{Models.}
We evaluate tokenwise speculative decoding using two target--draft model pairs. For the Qwen2.5 experiments in Table~\ref{tab:verification_result_qw25}, we use the instruction-tuned, 8-bit GPTQ-quantized~\citep{frantar2022gptq} models
\texttt{Qwen/Qwen2.5-72B-Instruct-GPTQ-Int8} and
\texttt{Qwen/Qwen2.5-0.5B-Instruct-GPTQ-Int8}
as the target and draft models, respectively~\citep{qwen2.5}.
For the LLaMA experiments in Table~\ref{tab:verification_result_llama31}, we use
\texttt{meta-llama/Llama-3.1-8B-Instruct} as the target model and
\texttt{meta-llama/Llama-3.2-1B-Instruct} as the draft model.

For the EAGLE-3 experiments~\citep{li2025eagle}, we consider two target--draft configurations. Table~\ref{tab:eagle_llama31} uses
\texttt{meta-llama/Llama-3.1-8B-Instruct} as the target model and
\texttt{yuhuili/EAGLE3-LLaMA3.1-Instruct-8B} as the corresponding EAGLE-3 draft model. Table~\ref{tab:eagle_qwen3} uses
\texttt{Qwen/Qwen3-8B} as the target model and
\texttt{Tengyunw/qwen3\_8b\_eagle3} as the EAGLE-3 draft model.

\begin{table*}[t]
\centering
\small
\caption{
Parameter counts of the target and draft models used in the experiments.
For EAGLE-3, the draft-model count denotes the additional parameters
of the EAGLE-3 draft head.
}
\label{tab:model_parameters}
\setlength{\tabcolsep}{5pt}
\renewcommand{\arraystretch}{1.15}
\begin{adjustbox}{max width=\textwidth}
\begin{tabular}{llllc}
\toprule
\textbf{Decoding Framework}
& \textbf{Model Pair}
& \textbf{Role}
& \textbf{Model}
& \textbf{Parameters} \\
\midrule

\multirow{4}{*}{Tokenwise SD}
& \multirow{2}{*}{Qwen2.5}
& Target
& \texttt{Qwen/Qwen2.5-72B-Instruct-GPTQ-Int8}
& 72.7B \\
&
& Draft
& \texttt{Qwen/Qwen2.5-0.5B-Instruct-GPTQ-Int8}
& 0.49B \\
\cmidrule(lr){2-5}
& \multirow{2}{*}{LLaMA}
& Target
& \texttt{meta-llama/Llama-3.1-8B-Instruct}
& 8B \\
&
& Draft
& \texttt{meta-llama/Llama-3.2-1B-Instruct}
& 1.23B \\

\midrule

\multirow{4}{*}{EAGLE-3}
& \multirow{2}{*}{LLaMA}
& Target
& \texttt{meta-llama/Llama-3.1-8B-Instruct}
& 8B \\
&
& Draft
& \texttt{yuhuili/EAGLE3-LLaMA3.1-Instruct-8B}
& $\sim$0.4B \\
\cmidrule(lr){2-5}
& \multirow{2}{*}{Qwen3}
& Target
& \texttt{Qwen/Qwen3-8B}
& 8.2B \\
&
& Draft
& \texttt{Tengyunw/qwen3\_8b\_eagle3}
& $\sim$0.4B \\

\bottomrule
\end{tabular}
\end{adjustbox}
\end{table*}

\paragraph{Artifact Licensing.}
We release only our source code and experimental scripts. Our implementation builds on publicly available open-source models, which remain subject to their original licenses and are not redistributed as part of our release. Users are required to obtain the corresponding model weights from their original repositories and comply with the associated license terms. We preserve all required attribution and license notices for reused components, and release only code that we have the right to distribute.

\paragraph{Computation hours.} For the Qwen2.5 experiments in Table~\ref{tab:verification_result_qw25}, it takes $\sim 5000$ A100 hours. For the LLaMA experiments in Table~\ref{tab:verification_result_llama31}, it takes $\sim 200$ A6000 hours. For the EAGLE-3 experiments in Table~\ref{tab:eagle_llama31}, it takes $\sim 100$ A6000 hours. For the EAGLE-3 experiments in Table~\ref{tab:eagle_qwen3}, it takes $\sim 100$ A6000 hours. 

\paragraph{Benchmarks and Metrics.}
We evaluate all methods on four benchmarks covering complementary generation capabilities: MATH~\cite{hendrycks2021measuring} for mathematical reasoning, MBPP+~\cite{liu2023your} for code generation, INCLUDE for multilingual knowledge, and BFCL~\citep{berkeley-function-calling-leaderboard} for tool use. We use the first 500 examples from MATH~\cite{hendrycks2021measuring}, the full MBPP+~\cite{liu2023your} test set containing 378 examples, five examples per language from INCLUDE~\citep{romanou2024include} for a total of 220 examples, and 200 calls from the \texttt{parallel\_multiple} split of BFCL-v3~\citep{berkeley-function-calling-leaderboard}. Generation quality is measured using accuracy for MATH, INCLUDE, and BFCL, and Pass@1 for MBPP+. Efficiency is measured using block efficiency (BE) and decoding speed (DS), where DS is reported in generated tokens per second.

\paragraph{Decoding Configuration.}
The EAGLE-3 experiments in Tables~\ref{tab:eagle_llama31} and~\ref{tab:eagle_qwen3} use a sampling temperature of $T=0.7$ and a block size of $7$. For each relaxation method, we sweep the corresponding verification parameter over the values reported in the tables. In the overshoot-capping ablation in Table~\ref{tab:overshoot_alone}, results are averaged over $\ell\in{0.2,0.4,0.6,0.8}$, with $p_{\mathrm{base}}=0.5$ used to define the truncation set $\mathcal{A}_{\Theta}$.

\paragraph{Hardware Platform.}
The experiments in Figure~\ref{fig:difficulty_trend} are conducted on a single NVIDIA H200 GPU with 140GB of memory. The EAGLE-3 experiments reported in Tables~\ref{tab:eagle_llama31} and~\ref{tab:eagle_qwen3} are conducted on a single NVIDIA RTX A6000 GPU with 48\,GB of memory. All remaining verification and ablation experiments are conducted on two NVIDIA A100 GPUs, each with 80\,GB of memory.

\section{Extended Analysis of the Collaborative Verification Trade-off}
\label{app:trade_off_extended}

\Cref{fig:trade_qwen} reports the Qwen2.5-72B/0.5B pair; \Cref{fig:trade_llama}
repeats the comparison with LLaMA-3.1-8B/3.2-1B. Each curve sweeps the relaxation
hyperparameter of one method (the lenience factor $\ell$, the interpolation weight $\lambda$
of CoS, and the judge threshold $\delta$ of AutoJudge). We organize the discussion around three observations.

\begin{figure*}[t]
    \centering
    \begin{subfigure}{0.9\textwidth}
        \centering
        \includegraphics[width=\textwidth]{lenience_cos_judge_tradeoff.pdf}
        \caption{}
        \label{fig:trade_qwen}
    \end{subfigure}
    \begin{subfigure}{0.9\textwidth}
        \centering
        \includegraphics[width=\textwidth]{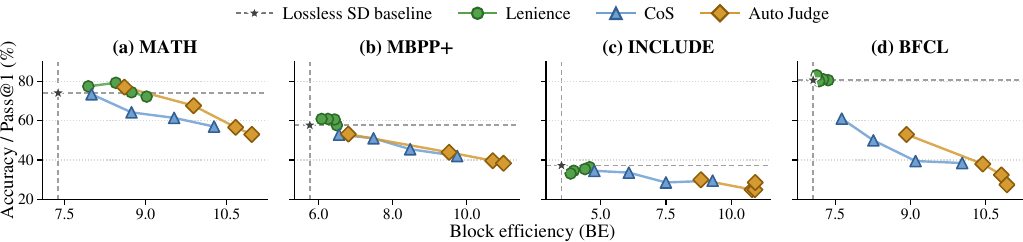}
        \caption{}
        \label{fig:trade_llama}
    \end{subfigure}

    \caption{
        Efficiency and task performance tradeoff for collaborative verification
        across four benchmarks.
        \Cref{fig:trade_qwen} uses Qwen2.5-72B-Instruct-GPTQ-Int8 as the target model
        and Qwen2.5-0.5B-Instruct-GPTQ-Int8 as the draft model, while
        \Cref{fig:trade_llama} uses Llama-3.1-8B-Instruct as the target model
        and Llama-3.2-1B-Instruct as the draft model.
    }
    \label{fig:trade_total}
\end{figure*}

\paragraph{Lenience preserves task performance across both model pairs.}
Across all four benchmarks and both model pairs, lenience stays within noise of the
lossless baseline in accuracy, in several cases marginally above it, while moving
block efficiency by only a small margin (roughly $+0.2$ to $+1.5$ BE). Its points are
also tightly clustered: sweeping $\ell$ over its usable range does not spread them out
the way $\lambda$ and the judge threshold $\delta$ spread the other two methods. Both
properties trace to the same source. The MBPP+ ablation in
\Cref{tab:ablation_yield_rule} isolates it: removing the overshoot ceiling while
retaining the adaptive interpolation coefficient recovers CoS-like degradation,
whereas retaining the ceiling alone preserves the flat segment. Capping the yield of
an overshooting token bounds the pointwise deviation from $p$, which is what protects
accuracy, and the same cap limits how much acceptance mass can be recovered, which is
why the efficiency gain is modest.

\paragraph{CoS degrades smoothly and monotonically with block efficiency.}
CoS traces a near-linear descent on every benchmark. This follows directly from
\Cref{eq:we_yield}: the induced distribution is a convex combination of $p$ and $q$
with a coefficient that grows with $\lambda$ uniformly over the vocabulary, so the
generated distribution drifts toward the draft at a rate set by a single free
scalar. Nothing in the rule distinguishes tokens whose substitution is harmless from
tokens that change the answer, and the empirical curves inherit that uniformity: the
trade-off is predictable, tunable, and strictly unfavorable.

\paragraph{The judge's apparent advantage is a domain effect, not an acceptance-rule effect.}
On MATH, AutoJudge dominates CoS at matched block efficiency for both model pairs, and
with the Qwen pair it sits \emph{above} the lossless baseline over a wide BE range.
On the remaining three benchmarks the two curves largely interleave, and both fall
well below the baseline once relaxation becomes aggressive. The size of the MATH gain
is also model-dependent: under the Qwen pair the judge clears the lossless baseline,
while under LLaMA-3.1-8B/3.2-1B the same sweep tracks at or slightly below it, though
still above CoS. On BFCL with the Qwen pair the sweep is non-monotone, with at least
one operating point scoring well below a neighbor at comparable block efficiency.

We attribute this pattern to the judge's supervision rather than to its acceptance
rule. Our judge is a single head trained on the union of mined GSM8K and
LiveCodeBench mismatches (\Cref{app:autojudge-implementation}), with GSM8K
contributing roughly ten times as many source problems, so the mixture is dominated
by mathematical reasoning. MATH is the one benchmark that matches the dominant
component of that mixture, and it is the one benchmark where the judge separates from
CoS. The residual model dependence is consistent with the mined labels being derived
from Qwen2.5 target completions, which aligns the judge with one of the two verified
models but not the other. The non-monotonicity on BFCL is consistent with $\delta$
thresholding a learned score whose calibration shifts across tasks, unlike $\ell$ and
$\lambda$, whose effects on the induced distribution are fixed in closed form by
\Cref{eq:lenience_yield,eq:we_yield}.

\section{Detailed Analysis in Overshoot Analysis in Lenience-based Relaxation}

The ablation in~\Cref{tab:overshoot_alone} supports our claim that adaptive interpolation is not the primary source of improvement in lenience-based relaxation: removing the adaptive interpolation while retaining the $p/\ell$ cap preserves most of the benefit. The overshoot region is precisely where the target and draft distributions disagree most strongly, and the draft tends to over-allocate probability to low-quality tokens there. Capping at $p/\ell$ in this region directly suppresses this failure mode.

We further investigate the overlap between the overshoot region and the set of tokens discarded by min-$p$ sampling, motivated by the intuition that accepting more tokens from the discarded set can improve efficiency. This suggests combining the two mechanisms: use the min-$p$ allowed set $\mathcal{A}_\Theta$ to gate acceptance, and apply the $p/\ell$ cap to control overshoot outside it. Concretely, tokens in $\mathcal{A}_\Theta$ are yielded with draft probability $q(x)$, while tokens outside are yielded with $\min\{q(x),\, p(x)/\ell\}$, so the rule relies on the draft distribution except in the overlap between the overshoot region and the discarded set. As shown in the last row of~\Cref{tab:overshoot_alone}, this rule delivers a $3.7\%$ gain in BE over lossless SD while matching its Pass@1 exactly, outperforming both SpecCascade and lenience-based relaxation. The discarded set thus appears to be a promising target for further acceleration at negligible quality cost.

\begin{table*}[t]
\centering
\footnotesize
\setlength{\tabcolsep}{5pt}
\renewcommand{\arraystretch}{1.2}
\caption{Effect of truncation and overshoot capping on MBPP+, averaged over $\ell \in \{0.2,0.4,0.6,0.8\}$. $p_\text{base}=0.5$ is used in $\mathcal{A}_\Theta$.}
\label{tab:overshoot_alone}
\begin{tabular}{p{9.0cm} c c}
\toprule
\textbf{Yielded rule} & \textbf{Avg. BE} & \textbf{Avg. Pass@1 (\%)} \\
\midrule
$\displaystyle
P(\text{generate}\,x)=p(x)
$ 
& 5.42 & 75.66 \\
\midrule
$\displaystyle
P(\text{generate}\,x)=
\begin{cases}
q(x), & x \in \mathcal{A}_\Theta,\\
0, & \text{otherwise},
\end{cases}
$ 
& 5.51 (+1.7\%) & 74.87 (-1.0\%) \\
\midrule
$\displaystyle
P(\text{generate}\,x)=
\begin{cases}
\Delta\, p(x) + (1-\Delta)\, q(x), & q(x)\le p(x),\\[2pt]
q(x), & p(x) < q(x) \le p(x)/\ell,\\[2pt]
p(x)/\ell, & q(x)\ge p(x)/\ell,
\end{cases}
$
& 5.53 (+2.0\%) & 75.22 (-0.6\%) \\
\midrule

$\displaystyle
P(\text{generate}\,x)=
\begin{cases}
q(x), & q(x)\le p(x)/\ell,\\[2pt]
p(x)/\ell, & q(x)\ge p(x)/\ell,
\end{cases}
$
& 5.58 (+3.0\%) & 75.33 (-0.4\%) \\
\midrule
$\displaystyle
P(\text{generate}\,x)=
\begin{cases}
q(x), & x \in \mathcal{A}_\Theta,\\
\min\{q(x),p(x)/\ell\},  & \text{otherwise}
\end{cases}
$
& \textbf{5.62 (+3.7\%)} & \textbf{75.66 (-0.0\%)} \\
\bottomrule
\end{tabular}
\end{table*}

\section{Extended Analysis for Judge-Based Relaxation}

\subsection{Judge-Based Relaxation Derivation}
\label{sec:judge_derivation}

Judge-based relaxation augments standard speculative verification with a
binary judgment signal $j_{\phi}(x)\in\{0,1\}$. When $q(x)>p(x)$, standard
speculative verification accepts the draft token with probability
$p(x)/q(x)$. If standard verification would reject the token, the judge
retains it whenever $j_{\phi}(x)=1$. Therefore, the effective acceptance
probability is
\begin{equation}
h(x)
=
\begin{cases}
1,
& q(x)\le p(x),\\[4pt]
\dfrac{p(x)}{q(x)}
+
\left(1-\dfrac{p(x)}{q(x)}\right)j_{\phi}(x),
& q(x)>p(x).
\end{cases}
\label{eq:judge_acceptance}
\end{equation}

As in standard speculative sampling, a rejected proposal is resampled
from the residual distribution
\begin{equation}
r(x)
=
\frac{
\bigl(p(x)-q(x)\bigr)_+
}{
\sum_{v\in\mathcal V}
\bigl(p(v)-q(v)\bigr)_+
}.
\label{eq:judge_residual}
\end{equation}

The resulting probability of generating token $x$ is
\begin{align}
P(\mathrm{generate}\,x)
&=
P(\mathrm{draft\ and\ accept}\,x)
+
P(\mathrm{reject})\,r(x)
\nonumber\\
&=
q(x)h(x)
+
\left[
\sum_{v\in\mathcal V}
q(v)\bigl(1-h(v)\bigr)
\right]
r(x).
\label{eq:judge_general_yield}
\end{align}

We first simplify the total rejection probability. For tokens satisfying
$q(v)\le p(v)$, we have $h(v)=1$, and hence they contribute no rejected
mass. For $q(v)>p(v)$,
\begin{align}
q(v)\bigl(1-h(v)\bigr)
&=
q(v)
\left[
1-\frac{p(v)}{q(v)}
-\left(1-\frac{p(v)}{q(v)}\right)j(v)
\right]
\nonumber\\
&=
\bigl(q(v)-p(v)\bigr)\bigl(1-j(v)\bigr).
\end{align}

Therefore,
\begin{equation}
R_j
\coloneqq
P(\mathrm{reject})
=
\sum_{v\in\mathcal V}
\bigl(q(v)-p(v)\bigr)_+
\bigl(1-j(v)\bigr).
\label{eq:judge_rejection_mass}
\end{equation}

Let
\begin{equation}
R
=
\sum_{v\in\mathcal V}
\bigl(p(v)-q(v)\bigr)_+.
\label{eq:standard_rejection_mass}
\end{equation}

Since $p$ and $q$ are normalized distributions,
\begin{equation}
\sum_{v\in\mathcal V}
\bigl(q(v)-p(v)\bigr)_+
=
\sum_{v\in\mathcal V}
\bigl(p(v)-q(v)\bigr)_+
=
R.
\label{eq:positive_mass_equality}
\end{equation}

Define
\begin{equation}
\Delta_j
=
\frac{R_j}{R}
=
\frac{
\sum_{v\in\mathcal V}
\bigl(q(v)-p(v)\bigr)_+
\bigl(1-j(v)\bigr)
}{
\sum_{v\in\mathcal V}
\bigl(p(v)-q(v)\bigr)_+
}.
\label{eq:judge_beta}
\end{equation}

Because $j(v)\in\{0,1\}$,
\begin{equation}
0
\le
\bigl(1-j(v)\bigr)
\le
1,
\end{equation}
which implies
\begin{equation}
0\le R_j\le R
\qquad\Longrightarrow\qquad
\beta\in[0,1].
\label{eq:judge_beta_range}
\end{equation}

We now derive the resulting distribution in the two regions.

\paragraph{Draft-underestimation region: $q(x)\le p(x)$.}
In this region, the draft token is always accepted, so
\begin{equation}
q(x)h(x)=q(x).
\end{equation}
Moreover,
\begin{equation}
r(x)
=
\frac{p(x)-q(x)}{R}.
\end{equation}
Substituting into \Cref{eq:judge_general_yield} gives
\begin{align}
P(\mathrm{generate}\,x)
&=
q(x)
+
R_j
\frac{p(x)-q(x)}{R}
\nonumber\\
&=
q(x)+\Delta_j\bigl(p(x)-q(x)\bigr)
\nonumber\\
&=
\Delta_j p(x)
+
(1-\Delta_j)q(x).
\label{eq:judge_underestimate_yield}
\end{align}

\paragraph{Draft-overshoot region: $q(x)>p(x)$.}
In this region, the residual distribution assigns zero probability to
$x$, since
\begin{equation}
\bigl(p(x)-q(x)\bigr)_+=0.
\end{equation}
Thus, the generated probability consists only of accepted draft mass:
\begin{align}
P(\mathrm{generate}\,x)
&=
q(x)h(x)
\nonumber\\
&=
q(x)
\left[
\frac{p(x)}{q(x)}
+
\left(1-\frac{p(x)}{q(x)}\right)j_{\phi}(x)
\right]
\nonumber\\
&=
p(x)+\bigl(q(x)-p(x)\bigr)j_{\phi}(x)
\nonumber\\
&=
\bigl(1-j_{\phi}(x)\bigr)p(x)
+
j_{\phi}(x)q(x).
\label{eq:judge_overshoot_yield}
\end{align}

Combining the two regions, the induced distribution is
\begin{equation}
\begin{gathered}
P(\mathrm{generate}\,x)=
\left\{
\begin{array}{@{}l@{\,}l@{}}
\Delta_j p(x) + \bigl(1-\Delta_j\bigr)q(x),
    \,\, q(x)\le p(x),\\[3pt]
p(x),
    \,\,\,\,\, q(x)>p(x)\ \text{and}\ j_{\phi}(x)=0,\\[3pt]
q(x), \,\,\,\,\,\,
     q(x)>p(x)\ \text{and}\ j_{\phi}(x)=1.
\end{array}
\right.
\end{gathered}
\label{eq:judge_yield}
\end{equation}
where $\Delta_j$ is given by \Cref{eq:judge_beta}.

Although all three methods introduce a learned judge to accept
target--draft mismatches, they differ primarily in how the judge
supervision is obtained. Judge Decoding~\citep{bachmann2025judge}
constructs a manually curated training set to teach a lightweight
verifier whether a draft continuation remains valid despite differing
from the target output. AutoJudge~\citep{garipov2025autojudge} removes
the need for manual annotation by automatically identifying
task-critical mismatches through counterfactual correction and
downstream outcome evaluation, but consequently relies on tasks with
verifiable quality signals. Thus, both methods share the same
judge-based relaxation mechanism at inference time, while progressing
from manually supervised, to outcome-supervised, to target-derived
self-supervised judge training.

\subsection{Judge-based Relaxation Implementation}
\label{app:autojudge-implementation}

We instantiate judge-based relaxation with AutoJudge~\citep{garipov2025autojudge},
using the authors' public implementation.\footnote{\url{https://github.com/garipovroma/autojudge}} Both models are frozen: \texttt{Qwen2.5-72B-Instruct-GPTQ-Int8} as the target model and \texttt{Qwen2.5-0.5B-Instruct-GPTQ-Int8} as the draft model.

AutoJudge labels a draft/target mismatch as \emph{unimportant} if substituting the
draft token and letting the target model complete the response preserves the final
answer, and as \emph{important} otherwise. The search is semi-greedy: it proceeds
from the earliest mismatch, keeps accepted substitutions in the running prefix, and
re-derives the remaining mismatches after each accepted swap, so that labels reflect
the joint effect of substitutions rather than each token in isolation. We run this
procedure on the 7{,}473-problem GSM8K~\citep{cobbe2021gsm8k} training split, with
answer equivalence defined by the extracted final number, and on the 880 problems of
the LiveCodeBench~\citep{jain2024livecodebench} \texttt{code\_generation\_lite}
\texttt{release\_v5} split, with equivalence defined by the benchmark's test suite. Unlike the original setup, which trains one classifier per task, we train a \emph{single} judge on the union of both sources, since our evaluation spans four benchmarks rather than the one the judge was mined on.

\subsection{Extended Analysis of Judge-based Verification}
\label{app:autojudge-domain}

\Cref{tab:mix_sweep_q_all} sweeps the composition of the judge's mining data
(GSM8K:LiveCodeBench, from math-dominated $9{:}1$ to code-dominated $1{:}9$)
against the acceptance threshold.

\paragraph{Source quality dominates the efficiency-performance tradeoff.}
Degradation from $9{:}1$ to $1{:}9$ at $ 0.8$ orders as MATH ($-14.0$),
MBPP+ ($-29.6$), BFCL ($-32.0$), INCLUDE ($-40.0$), consistent with proximity to
GSM8K, the source that survives at high math ratios.

\begin{table*}[t]
  \centering
    \caption{Mixture-ratio sweep, Qwen2.5-72B / 0.5B, thresholds matched at the
    20th/40th/60th/80th percentile of each head's validation scores. Single run per cell.}
  \label{tab:mix_sweep_q_all}
  \resizebox{\linewidth}{!}{%
  \begin{tabular}{l|c|cc|cc|cc|cc}
    \toprule
    \multirow{2}{*}{\textbf{Mix}}
      & \multirow{2}{*}{\textbf{Param}}
      & \multicolumn{2}{c|}{\textbf{MATH}}
      & \multicolumn{2}{c|}{\textbf{MBPP+}}
      & \multicolumn{2}{c|}{\textbf{INCLUDE}}
      & \multicolumn{2}{c}{\textbf{BFCL}}
      \\
    \cmidrule(lr){3-4}\cmidrule(lr){5-6}\cmidrule(lr){7-8}\cmidrule(lr){9-10}
      & & BE & Acc\,(\%)
      & BE & Pass@1\,(\%)
      & BE & Acc\,(\%)
      & BE & Acc\,(\%) \\
    \midrule
    \multirow{4}{*}{9:1} & $ 0.2$
            & $8.59$ & $84.20$
            & $5.81$ & $76.72$
            & $3.76$ & $55.91$
            & $8.88$ & $88.00$ \\
     & $ 0.4$
            & $8.67$ & $82.40$
            & $5.81$ & $76.72$
            & $3.89$ & $65.91$
            & $8.94$ & $88.50$ \\
     & $ 0.6$
            & $8.47$ & $83.80$
            & $5.81$ & $76.72$
            & $3.93$ & $60.00$
            & $8.84$ & $87.50$ \\
     & $ 0.8$
            & $8.76$ & $82.00$
            & $5.83$ & $76.72$
            & $4.00$ & $70.00$
            & $8.88$ & $88.00$ \\
    \midrule
    \multirow{4}{*}{7:3} & $ 0.2$
            & $8.51$ & $82.60$
            & $5.81$ & $76.72$
            & $3.90$ & $70.00$
            & $8.78$ & $88.50$ \\
     & $ 0.4$
            & $9.04$ & $82.40$
            & $5.82$ & $76.19$
            & $4.56$ & $70.00$
            & $8.94$ & $87.00$ \\
     & $ 0.6$
            & $10.17$ & $68.20$
            & $6.52$ & $70.90$
            & $7.33$ & $48.18$
            & $9.75$ & $82.00$ \\
     & $ 0.8$
            & $10.97$ & $62.60$
            & $8.46$ & $55.56$
            & $10.22$ & $48.18$
            & $10.09$ & $52.00$ \\
    \midrule
    \multirow{4}{*}{5:5} & $ 0.2$
            & $8.95$ & $78.40$
            & $6.20$ & $69.31$
            & $5.73$ & $54.09$
            & $8.70$ & $88.00$ \\
     & $ 0.4$
            & $9.31$ & $74.40$
            & $6.93$ & $61.90$
            & $6.29$ & $45.91$
            & $9.14$ & $84.00$ \\
     & $ 0.6$
            & $9.98$ & $76.60$
            & $7.83$ & $58.20$
            & $7.73$ & $40.00$
            & $9.97$ & $70.00$ \\
     & $ 0.8$
            & $10.45$ & $68.80$
            & $9.42$ & $51.06$
            & $8.74$ & $50.00$
            & $10.30$ & $66.00$ \\
    \midrule
    \multirow{4}{*}{3:7} & $ 0.2$
            & $8.58$ & $80.40$
            & $5.86$ & $76.46$
            & $4.10$ & $61.82$
            & $8.86$ & $88.00$ \\
     & $ 0.4$
            & $9.22$ & $82.60$
            & $5.93$ & $77.78$
            & $5.46$ & $60.00$
            & $9.69$ & $88.00$ \\
     & $ 0.6$
            & $10.08$ & $78.40$
            & $6.58$ & $72.22$
            & $7.85$ & $51.82$
            & $9.44$ & $84.00$ \\
     & $ 0.8$
            & $10.94$ & $52.20$
            & $8.26$ & $60.05$
            & $10.49$ & $30.00$
            & $9.98$ & $70.00$ \\
    \midrule
    \multirow{4}{*}{1:9} & $ 0.2$
            & $9.73$ & $72.60$
            & $6.28$ & $70.90$
            & $6.16$ & $44.09$
            & $8.70$ & $76.00$ \\
     & $ 0.4$
            & $10.57$ & $64.40$
            & $7.63$ & $60.32$
            & $8.77$ & $58.18$
            & $9.96$ & $62.00$ \\
     & $ 0.6$
            & $10.90$ & $52.80$
            & $8.94$ & $53.17$
            & $10.05$ & $41.82$
            & $10.42$ & $58.00$ \\
     & $ 0.8$
            & $10.99$ & $68.00$
            & $10.51$ & $47.09$
            & $10.57$ & $30.00$
            & $10.61$ & $56.00$ \\
    \bottomrule[1.3pt]
  \end{tabular}}
\end{table*}

\section{Block Efficiency and Distributional Gap under Tree Verification}
\label{app:be-tree}

This section proves \Cref{lem:trunc-be} (\Cref{app:proof-lemma-be}), and proves the distributional gap of \Cref{thm:kl-gap} and the block efficiency of Proposition~\ref{prop:be-tree} through a per-position analysis (Lemma~\ref{lem:node}, Corollary~\ref{cor:multidraft-gap}). Throughout, $p$, $q$, $p_\Theta$, $z$, $Z_\Theta$, and $\mathcal{A}_\Theta$ refer to the current decoding step and are conditioned on all previously generated tokens, as in \Cref{sec:truncation_sampling}.

\subsection{Proof of Truncation sampling efficiency effect}
\label{app:proof-lemma-be}

\begin{proof}
Fix a decoding step and condition on the prefix. Under standard speculative
decoding with draft $q$ and reference distribution $\tilde p$, the single-token
acceptance probability is
\begin{equation}
A(\tilde p, q)
= \sum_{x \in \mathcal{V}} q(x) \min\Big\{1, \frac{\tilde p(x)}{q(x)}\Big\}
= \sum_{x \in \mathcal{V}} \min\big(\tilde p(x), q(x)\big).
\label{eq:acc-prob}
\end{equation}
Without truncation $\tilde p = p$; with truncation strategy $\Theta$ applied to
the target, $\tilde p = p_\Theta$, i.e.
$p_\Theta(x) = p(x)/z$ for $x \in \mathcal{A}_\Theta$ and $0$ otherwise, where
$z = \sum_{v \in \mathcal{A}_\Theta} p(v)$. The draft distribution is
unmodified. Hence
\begin{equation}
\Delta\mathrm{BE}
= A(p_\Theta, q) - A(p, q)
= \sum_{x \in \mathcal{A}_\Theta}
  \Big[ \min\big(\tfrac{p(x)}{z}, q(x)\big) - \min\big(p(x), q(x)\big) \Big]
- \sum_{x \notin \mathcal{A}_\Theta} \min\big(p(x), q(x)\big),
\label{eq:be-split}
\end{equation}
since $p_\Theta$ vanishes off $\mathcal{A}_\Theta$.

It remains to evaluate the bracketed term for $x \in \mathcal{A}_\Theta$.
Because $z \leq 1$ we have $p(x)/z \geq p(x)$, and there are two cases.

\paragraph{Case $q(x) \leq p(x)$.}
Then $\min(p(x), q(x)) = q(x)$ and $\min(p(x)/z, q(x)) = q(x)$, so the
bracketed term is $0$.

\paragraph{Case $q(x) > p(x)$.}
Then $\min(p(x), q(x)) = p(x)$, and
\begin{equation}
\min\Big(\frac{p(x)}{z}, q(x)\Big) - p(x)
= \min\Big(\frac{p(x)}{z} - p(x),\, q(x) - p(x)\Big)
= \min\Big(\Big(\frac{1}{z} - 1\Big) p(x),\, q(x) - p(x)\Big).
\end{equation}

Both cases are covered by
$\min\{(q(x) - p(x))_+,\, (\tfrac{1}{z} - 1) p(x)\}$, which vanishes when
$q(x) \leq p(x)$. Substituting into Equation~\eqref{eq:be-split} gives
\begin{equation}
\Delta\mathrm{BE}
= \sum_{x \in \mathcal{A}_\Theta}
  \min\Big\{ \big(q(x) - p(x)\big)_+,\, \Big(\frac{1}{z} - 1\Big) p(x) \Big\}
- \sum_{x \notin \mathcal{A}_\Theta} \min\big(p(x), q(x)\big),
\end{equation}
as claimed.
\end{proof}

\subsection{Tree verification as implemented}
\label{app:tree-impl}
EAGLE-3 drafts a token tree and verification accepts a single root-to-leaf path~\citep{li2025eagle}. Greedy drafting fixes a deterministic candidate set $x^{(1)},\dots,x^{(D)}$ at each position ($D$ fixed by the drafting configuration), examined in this order, with draft probabilities playing no role in verification. With a truncation warper active the verifier scores candidates by the truncated target $p_\Theta$ (\Cref{eq:trunc_verif_pyield}), accepting each iff $r\le h(x^{(i)})$, $r\sim\mathrm{Unif}[0,1]$:
\begin{itemize}[leftmargin=1.6em]
\item \textbf{Truncation sampling + SD (lossless).}
\begin{equation}
\label{eq:h-lossless}
h(x)=\min\!\Bigl\{p_\Theta(x)/q,\,1\Bigr\}\big|_{q=1}=p_\Theta(x),
\end{equation}
with zero-and-renormalize on rejection, reproducing $p_\Theta$ at each position (Lemma~\ref{lem:node}(i)).
\item \textbf{SpecCascade.}
\begin{equation}
\label{eq:h-cascade}
\begin{gathered}
h(x)=\mathbf 1\!\left[p_\Theta(x)\ge p_{\text{base}}\,\max_{v}p_\Theta(v)\right]
=\mathbf 1[x\in\mathcal{A}_{\text{min-}p}].
\end{gathered}
\end{equation}
\item \textbf{Typical acceptance.}
\begin{equation}
\label{eq:h-typical}
\begin{gathered}
h(x)=\min\!\Bigl\{p_\Theta(x)/\tau,\,1\Bigr\}=\mathbf 1[x\in\mathcal{A}_\eta],\qquad
\tau=\min\!\bigl(\varepsilon,\,\sqrt\varepsilon\,e^{-H(p_\Theta)}\bigr).
\end{gathered}
\end{equation}
\end{itemize}
Equations~\ref{eq:h-cascade} and~\ref{eq:h-typical} are the indicator of Definition~\ref{def:truncation-verification} with $\Theta=$ min-$p$ (\Cref{eq:min-p}) and $\Theta=\eta$ (\Cref{eq:eta}, $\delta=\sqrt\varepsilon$): off-set candidates are rejected and in-set candidates accepted with probability one. If a candidate is accepted the walk descends into its subtree; otherwise the path terminates with a bonus token from $p_\Theta$. The accepted length $L$ counts accepted drafted tokens only (the bonus excluded) and equals the block efficiency in \Cref{tab:eagle_llama31,tab:eagle_qwen3}.

\subsection{Per-position analysis}

Write $a(C)=\Pr[\text{some }x\in C\text{ accepted}]$ for the per-position acceptance probability of a rule, given the prefix.

\begin{lemma}[]\textbf{Per-position generation and acceptance}
\label{lem:node}
Let $C=\{x^{(1)},\dots,x^{(D)}\}$ be the candidates at the current position, examined in order, with $p_\Theta(C)=\sum_{x\in C}p_\Theta(x)$. Under \Cref{app:tree-impl}:
\begin{enumerate}[label=(\roman*),leftmargin=2.2em]
\item Truncation sampling + SD emits $p_\Theta$ exactly, with
\begin{equation}
\label{eq:node-lossless}
\begin{gathered}
\Pr[\text{accept }x^{(i)}]=p_\Theta(x^{(i)}),
a^{\mathrm{lossless}}(C)=p_\Theta(C).
\end{gathered}
\end{equation}
\item Any rule with $\tilde p=p_\Theta$ accepts only tokens in $C\cap\mathcal{A}_\Theta$, with
\begin{equation}
\label{eq:node-lossy}
\begin{gathered}
a^{\mathrm{lossy}}(C)=\mathbf 1[C\cap\mathcal{A}_\Theta\neq\varnothing]\ \text{(SpecCasc.)};
a^{\mathrm{lossy}}(C)\in(0,1],\ \text{with }{>}0\\
\iff C\cap\mathcal{A}_\Theta\neq\varnothing\ \ \text{(typical acc.)}.
\end{gathered}
\end{equation}
\end{enumerate}
\end{lemma}

\begin{proof}
\emph{(i)} Initialize $\tilde p_1=p_\Theta$; the zero-and-renormalize update gives $\tilde p_i=p_\Theta(\,\cdot\mid\mathcal V\setminus\{x^{(1)},\dots,x^{(i-1)}\})$. Then
\begin{align*}
\Pr[\text{reach }x^{(i)}]
&=\prod_{k=1}^{i-1}\bigl(1-\tilde p_k(x^{(k)})\bigr)\\
&=\prod_{k=1}^{i-1}\frac{1-\sum_{l\le k}p_\Theta(x^{(l)})}{1-\sum_{l\le k-1}p_\Theta(x^{(l)})}\\
&=1-\sum_{l<i}p_\Theta(x^{(l)}),\\
\Pr[\text{accept }x^{(i)}]
&=\Pr[\text{reach }x^{(i)}]\cdot\tilde p_i(x^{(i)})\\
&=\Bigl(1-\!\textstyle\sum_{l<i}p_\Theta(x^{(l)})\Bigr)\times\dfrac{p_\Theta(x^{(i)})}{1-\sum_{l<i}p_\Theta(x^{(l)})}\\
&=p_\Theta(x^{(i)}),
\end{align*}
\[\Rightarrow a^{\mathrm{lossless}}(C)=\sum_i p_\Theta(x^{(i)})=p_\Theta(C)\]. On all-reject,
\[
\Pr[\text{emit }x]=\bigl(1-p_\Theta(C)\bigr)\frac{p_\Theta(x)}{1-p_\Theta(C)}=p_\Theta(x),
\]
for $x\notin C$,
so the position emits $p_\Theta$ on all of $\mathcal V$.

\emph{(ii)} With $\tilde p=p_\Theta$, \Cref{eq:h-cascade} and \Cref{eq:h-typical} give
\[
x\notin\mathcal{A}_\Theta\ \Rightarrow\ \tilde p(x)=0\ \Rightarrow\ h(x)=0,
\]
so accepted tokens lie in $C\cap\mathcal{A}_\Theta$. SpecCascade has $h\equiv 1$ on $\mathcal{A}_\Theta$ ($\Rightarrow$ first in-set candidate accepted surely); typical acceptance has $h(x)=\min\{p_\Theta(x)/\tau,1\}>0$ on $\mathcal{A}_\Theta$. This is \Cref{eq:node-lossy}.
\end{proof}

\begin{corollary}[]\textbf{Tree verification accepts at least as often}
\label{cor:multidraft-gap}
At every position,
\begin{equation}
\label{eq:be-domination}
\begin{gathered}
a^{\mathrm{lossless}}(C)=p_\Theta(C)=\frac{\sum_{v\in C\cap\mathcal{A}_\Theta}p(v)}{z}
\le\ \mathbf 1[C\cap\mathcal{A}_\Theta\neq\varnothing]=a^{\mathrm{lossy}}_{\mathrm{SpecCascade}}(C),
\end{gathered}
\end{equation}
strict whenever $0<p_\Theta(C)<1$; for typical acceptance $a^{\mathrm{lossless}}\le a^{\mathrm{lossy}}$ holds accumulated along the path. 
\end{corollary}

\begin{proof}
By Lemma~\ref{lem:node}, $a^{\mathrm{lossless}}(C)=p_\Theta(C)\in[0,1]$ and $p_\Theta(C)>0\iff C\cap\mathcal{A}_\Theta\neq\varnothing$, so $p_\Theta(C)\le\mathbf 1[C\cap\mathcal{A}_\Theta\neq\varnothing]$, with equality iff $p_\Theta(C)\in\{0,1\}$.
\end{proof}


\subsection{Proof of Draft-target divergence under tree verification theorem}
\begin{proof}
Write $p_\Theta(x)=p(x)/Z_\Theta(p)$ for $x\in\mathcal{A}_\Theta$ (zero otherwise), the truncated target; by Lemma~\ref{lem:node}(i) truncation sampling emits $p_\Theta$ at every position.

\emph{Statement 1 ($\mathrm{KL}_{\mathrm{SD}}$).} Truncation-based verification under standard SD generates $q/Z_\Theta(q)$ on $\mathcal{A}_\Theta$ (\Cref{eq:trunc_verif_pyield}). The per-token term is
\[
\begin{aligned}
D_{\mathrm{KL}}\!\bigl(q/Z_\Theta(q)\,\big\|\,p_\Theta\bigr)
&=\sum_{x\in\mathcal{A}_\Theta}\frac{q(x)}{Z_\Theta(q)}\log\frac{q(x)/Z_\Theta(q)}{p(x)/Z_\Theta(p)}\\
&=\sum_{x\in\mathcal{A}_\Theta}\frac{q(x)}{Z_\Theta(q)}\log\frac{q(x)\,Z_\Theta(p)}{Z_\Theta(q)\,p(x)},
\end{aligned}
\]
the $\mathrm{KL}_{\mathrm{SD}}$ of~\Cref{thm:kl-gap}, finite since $p>0$ on $\mathcal{A}_\Theta$, and
\[
\begin{gathered}
q=p\ \Rightarrow\ Z_\Theta(q)=Z_\Theta(p)
\Rightarrow\ q/Z_\Theta(q)=p_\Theta\ \Rightarrow\ D_{\mathrm{KL}}=0,
\end{gathered}
\]
so $\mathrm{KL}_{\mathrm{SD}}\to 0$ as $q\to p$ (\Cref{thm:kl-gap}).

\emph{Statement 2 ($\mathrm{KL}_{\mathrm{EAGLE}}$).} By Lemma~\ref{lem:node}(ii) the accepted token is the first candidate in $\mathcal{A}_\Theta$, deterministic given the prefix, with $p_\Theta(x)=p(x)/Z_\Theta(p)>0$ and the accept test governed by $p_\Theta$, not $q$. The position therefore emits the point mass at $x$, whose divergence from $p_\Theta$ is
\[
\log\frac{1}{p_\Theta(x)}=\log\frac{Z_\Theta(p)}{p(x)}
\;=\;\mathrm{KL}_{\mathrm{EAGLE}},
\]
for SpecCascade and typical acceptance alike. Finally,
\[
\begin{gathered}
|\mathcal{A}_\Theta|>1\ \Rightarrow\ p_\Theta(x)\le\max_{v\in\mathcal{A}_\Theta}p_\Theta(v)<1
\Rightarrow\ \log\frac{Z_\Theta(p)}{p(x)}\ge\log\frac{1}{\max_{v\in\mathcal{A}_\Theta}p_\Theta(v)}>0,
\end{gathered}
\]
for every draft and uniformly in $q$, so $\mathrm{KL}_{\mathrm{EAGLE}}>0$ does not vanish as $q\to p$ (\Cref{thm:kl-gap}).
\end{proof}

\subsection{Block Efficiency Analysis}

\begin{proposition}[]\textbf{Block efficiency of tree verification}
\label{prop:be-tree}
With $L$ as above,
\begin{equation}
\label{eq:be-multidraft}
\mathrm{BE}=\mathbb{E}[L]=\sum_{d\ge 1}\ \prod_{j=1}^{d}\Pr[L\ge j\mid L\ge j-1],
\end{equation}

\end{proposition}

In \Cref{eq:be-multidraft}, every factor is at least as large under truncation-based verification as under truncation sampling + SD; hence its block efficiency is weakly higher. \emph{(See \Cref{app:be-tree} for the proof.)} However, the slight gain in BE is at the cost of even more severe performance degradation under multi-draft setting. 

\begin{proof}[Proof of Proposition~\ref{prop:be-tree}]
\[
\begin{gathered}
\mathrm{BE}=\mathbb{E}[L]=\sum_{d\ge1}\Pr[L\ge d]
=\sum_{d\ge1}\prod_{j=1}^{d}\Pr[L\ge j\mid L\ge j-1],
\end{gathered}
\]
by \Cref{eq:be-multidraft}, each factor being $a(C)$ of Lemma~\ref{lem:node}. By Corollary~\ref{cor:multidraft-gap}, $a^{\mathrm{lossy}}\ge a^{\mathrm{lossless}}$ (per-position and strict for SpecCascade when $0<p_\Theta(C)<1$; accumulated along the path for typical acceptance), so every factor is weakly larger and $\mathrm{BE}$ is weakly higher.
\end{proof}

\subsection{Qualitative Examples}
\label{sec:appendix-qual}

We hereby include some samples revealing the failure mode of truncation-based verifications against their matched baselines.
\begin{tcolorbox}[
  colback=gray!5, colframe=gray!50, arc=2pt,
  left=4pt,right=4pt,top=4pt,bottom=4pt,
  boxrule=0.4pt, fontupper=\small,
  label={box:mbpp9},
  title={\small MBPP+ task~$9$ --- wrong sentinel return},
  fonttitle=\small\bfseries
]
\textbf{Prompt:} \emph{Write a python function to find the minimum
number of rotations (greater than $0$) required to get the same
string.}\\[5pt]
Matched baseline (min-$p$):\\
\quad\texttt{for i in range(1, n):}\\
\quad\texttt{\ \ \ \ if temp[i:i+n] == s: return i}\\
\quad\texttt{return \textbf{n}}
\hfill {\color{green!50!black}\checkmark}\\[3pt]
SpecCascade (min-$p$ allowed set):\\
\quad\texttt{for i in range(1, n):}\\
\quad\texttt{\ \ \ \ if temp[i:i+n] == s: return i}\\
\quad\texttt{return \textbf{0}}
\hfill {\color{red}$\times$}\\[2pt]
\emph{Fingerprint:} both versions share identical body and loop logic;
they disagree only on the post-loop sentinel. The draft defaults to
\texttt{0}, which violates the ``greater than $0$'' specification; the
target returns \texttt{n}, the trivial full-rotation period.
\end{tcolorbox}

\begin{tcolorbox}[
  colback=gray!5, colframe=gray!50, arc=2pt,
  left=4pt,right=4pt,top=4pt,bottom=4pt,
  boxrule=0.4pt, fontupper=\small,
  label={box:bfcl195},
  title={\small BFCL \texttt{parallel\_multiple\_195} ---
         math vs.\ Python notation},
  fonttitle=\small\bfseries
]
\textbf{Ground-truth slot:}
\texttt{calculate\_area\_under\_curve(function=\textbf{"x**2"},\ldots)}\\[3pt]
Matched baseline (min-$p$):
\texttt{function=\textbf{"x**2"}}
\hfill {\color{green!50!black}\checkmark}\\
SpecCascade  (min-$p$ allowed set):
\texttt{function=\textbf{"x\^{}2"}}
\hfill {\color{red}$\times$}\\[2pt]
\emph{Fingerprint:} the draft writes the math-class form
\texttt{x\^{}2}; the target writes the Python form \texttt{x**2} that
the BFCL grader requires. This divergence recurs $14$ times across
the $(p_\text{base},\,\text{seed})$ grid, always in the same direction.
\end{tcolorbox}

\begin{tcolorbox}[
  colback=gray!5, colframe=gray!50, arc=2pt,
  left=4pt,right=4pt,top=4pt,bottom=4pt,
  boxrule=0.4pt, fontupper=\small,
  label={box:mbpp637},
  title={\small MBPP+ task~$637$ --- function signature mishread},
  fonttitle=\small\bfseries
]
\textbf{Prompt:} \emph{Write a function to check whether the given
amount has no profit and no loss.}\\[3pt]
\textbf{Hidden tests} call the function as
\texttt{noprofit\_noloss(cost\_price, selling\_price)}.\\[5pt]
Matched baseline (min-$p$):\\
\quad\texttt{def noprofit\_noloss(\textbf{cost\_price, selling\_price}):}\\
\quad\texttt{\ \ \ \ return cost\_price == selling\_price}
\hfill {\color{green!50!black}\checkmark}\\[3pt]
SpecCascade (min-$p$ allowed set):\\
\quad\texttt{def noprofit\_noloss(\textbf{amount}):}\\
\quad\texttt{\ \ \ \ return amount == 0}
\hfill {\color{red}$\times$}\\[2pt]
\emph{Fingerprint:} the draft latches onto the salient word
\emph{``amount''} in the prompt and emits a plausible-looking
single-argument signature; the target reads the two-quantity nature of
the problem and emits the matching two-argument signature.
\end{tcolorbox}

\section{Extended Results}
\label{sec: total_results}

\begin{table*}[t]
  \centering
    \caption{
    Verification results across four benchmarks: MATH, MBPP+, INCLUDE, and
    BFCL. The target and draft models are
    Qwen/Qwen2.5-72B-Instruct-GPTQ-Int8 and
    Qwen/Qwen2.5-0.5B-Instruct-GPTQ-Int8, respectively.
    }
  \resizebox{\linewidth}{!}{%
  \begin{tabular}{l|c|ccc|ccc|ccc|ccc}
    \toprule
    \multirow{2}{*}{\textbf{Method}}
      & \multirow{2}{*}{\textbf{Param}}
      & \multicolumn{3}{c|}{\textbf{MATH}}
      & \multicolumn{3}{c|}{\textbf{MBPP+}}
      & \multicolumn{3}{c|}{\textbf{INCLUDE}}
      & \multicolumn{3}{c}{\textbf{BFCL}} \\
    \cmidrule(lr){3-5}\cmidrule(lr){6-8}\cmidrule(lr){9-11}\cmidrule(lr){12-14}
      & & BE & DS & Acc\,(\%)
      & BE & DS & Pass@1\,(\%)
      & BE & DS & Acc\,(\%)
      & BE & DS & Acc\,(\%) \\
    \midrule
SD Baseline~\cite{leviathan2023fast} & --
            & $7.98_{\pm0.02}$ & $4.67_{\pm0.02}$ & $76.47_{\pm1.53}$
            & $5.47_{\pm0.06}$ & $4.06_{\pm0.03}$ & $75.84_{\pm0.40}$
            & $3.40_{\pm0.03}$ & $4.63_{\pm0.01}$ & $68.18_{\pm0.00}$
            & $8.73_{\pm0.01}$ & $6.86_{\pm0.00}$ & $88.17_{\pm0.58}$ \\
    \midrule
    \multirow{5}{*}{\shortstack{Min-p Smpl.
    ~\citep{nguyen2025turningheatminpsampling} \\ + SD}}
      & 0.1 & $7.94_{\pm0.02}$ & $4.65_{\pm0.02}$ & $76.27_{\pm1.53}$
            & $5.42_{\pm0.08}$ & $3.89_{\pm0.14}$ & $75.57_{\pm0.40}$
            & $3.34_{\pm0.04}$ & $4.63_{\pm0.01}$ & $68.79_{\pm0.69}$
            & $8.73_{\pm0.01}$ & $6.97_{\pm0.18}$ & $88.17_{\pm0.58}$ \\
      & 0.3 & $7.99_{\pm0.05}$ & $4.64_{\pm0.01}$ & $76.67_{\pm0.83}$
            & $5.56_{\pm0.03}$ & $3.90_{\pm0.14}$ & $76.19_{\pm0.46}$
            & $3.42_{\pm0.04}$ & $4.62_{\pm0.01}$ & $68.18_{\pm2.08}$
            & $8.85_{\pm0.01}$ & $6.94_{\pm0.21}$ & $88.33_{\pm1.04}$ \\
      & 0.5 & $8.03_{\pm0.02}$ & $4.64_{\pm0.01}$ & $76.07_{\pm1.17}$
            & $5.45_{\pm0.07}$ & $3.89_{\pm0.14}$ & $75.57_{\pm0.31}$
            & $3.44_{\pm0.02}$ & $4.61_{\pm0.02}$ & $68.48_{\pm1.46}$
            & $8.88_{\pm0.02}$ & $6.97_{\pm0.19}$ & $87.67_{\pm0.29}$ \\
      & 0.7 & $8.08_{\pm0.04}$ & $4.64_{\pm0.01}$ & $76.27_{\pm0.42}$
            & $5.59_{\pm0.02}$ & $3.88_{\pm0.13}$ & $76.01_{\pm0.31}$
            & $3.47_{\pm0.05}$ & $4.63_{\pm0.01}$ & $66.52_{\pm1.14}$
            & $8.89_{\pm0.02}$ & $6.96_{\pm0.19}$ & $87.67_{\pm0.29}$ \\
      & 0.9 & $8.08_{\pm0.10}$ & $4.65_{\pm0.01}$ & $77.27_{\pm1.27}$
            & $5.62_{\pm0.04}$ & $3.88_{\pm0.14}$ & $76.01_{\pm0.15}$
            & $3.47_{\pm0.04}$ & $4.62_{\pm0.01}$ & $66.97_{\pm1.31}$
            & $8.89_{\pm0.01}$ & $6.97_{\pm0.18}$ & $88.00_{\pm0.00}$ \\
    \midrule
    \multirow{5}{*}{\shortstack{SpecCascade~\citep{narasimhan2024faster}}}
      & 0.1 & $8.46_{\pm0.06}$ & $4.53_{\pm0.22}$ & $76.87_{\pm0.46}$
            & $5.64_{\pm0.03}$ & $3.89_{\pm0.14}$ & $76.46_{\pm0.26}$
            & $3.57_{\pm0.04}$ & $4.63_{\pm0.01}$ & $64.85_{\pm1.14}$
            & $8.87_{\pm0.04}$ & $6.83_{\pm0.05}$ & $89.00_{\pm0.50}$ \\
      & 0.3 & $8.36_{\pm0.04}$ & $4.65_{\pm0.01}$ & $75.47_{\pm0.70}$
            & $5.59_{\pm0.04}$ & $3.89_{\pm0.14}$ & $75.40_{\pm0.26}$
            & $3.50_{\pm0.03}$ & $4.60_{\pm0.01}$ & $66.67_{\pm3.44}$
            & $8.89_{\pm0.01}$ & $6.81_{\pm0.02}$ & $88.67_{\pm0.29}$ \\
      & 0.5 & $8.12_{\pm0.05}$ & $4.66_{\pm0.01}$ & $74.87_{\pm0.42}$
            & $5.51_{\pm0.03}$ & $3.90_{\pm0.12}$ & $74.87_{\pm0.70}$
            & $3.40_{\pm0.08}$ & $4.61_{\pm0.02}$ & $68.33_{\pm3.03}$
            & $8.85_{\pm0.03}$ & $6.88_{\pm0.03}$ & $88.00_{\pm0.00}$ \\
      & 0.7 & $8.01_{\pm0.06}$ & $4.66_{\pm0.01}$ & $75.47_{\pm0.70}$
            & $5.48_{\pm0.01}$ & $3.90_{\pm0.14}$ & $76.46_{\pm0.53}$
            & $3.30_{\pm0.04}$ & $4.60_{\pm0.01}$ & $68.33_{\pm1.05}$
            & $8.84_{\pm0.03}$ & $6.82_{\pm0.03}$ & $87.67_{\pm0.29}$ \\
      & 0.9 & $7.84_{\pm0.07}$ & $4.66_{\pm0.01}$ & $75.47_{\pm1.40}$
            & $5.41_{\pm0.01}$ & $3.89_{\pm0.14}$ & $76.19_{\pm0.26}$
            & $3.25_{\pm0.03}$ & $4.62_{\pm0.01}$ & $65.91_{\pm0.45}$
            & $8.84_{\pm0.03}$ & $5.28_{\pm0.04}$ & $88.17_{\pm0.58}$ \\
    \midrule
    \multirow{5}{*}{\shortstack{$\eta$ Smpl.~\citep{hewitt2022truncation} \\ + SD}}
      & 0.05 & $7.92_{\pm0.04}$ & $4.63_{\pm0.01}$ & $76.67_{\pm1.85}$
            & $5.42_{\pm0.08}$ & $3.89_{\pm0.14}$ & $75.57_{\pm0.40}$
            & $3.34_{\pm0.05}$ & $4.61_{\pm0.02}$ & $67.73_{\pm1.82}$
            & $8.73_{\pm0.01}$ & $6.95_{\pm0.18}$ & $88.17_{\pm0.58}$ \\
      & 0.10 & $7.94_{\pm0.03}$ & $4.64_{\pm0.01}$ & $76.47_{\pm1.03}$
            & $5.49_{\pm0.06}$ & $3.88_{\pm0.15}$ & $76.54_{\pm0.61}$
            & $3.31_{\pm0.02}$ & $4.62_{\pm0.02}$ & $68.79_{\pm0.26}$
            & $8.74_{\pm0.01}$ & $6.82_{\pm0.02}$ & $88.17_{\pm0.58}$ \\
      & 0.15 & $7.96_{\pm0.03}$ & $4.64_{\pm0.02}$ & $75.87_{\pm2.21}$
            & $5.44_{\pm0.02}$ & $3.89_{\pm0.15}$ & $76.01_{\pm0.31}$
            & $3.39_{\pm0.05}$ & $4.60_{\pm0.03}$ & $67.27_{\pm0.45}$
            & $8.77_{\pm0.01}$ & $6.79_{\pm0.05}$ & $88.17_{\pm0.58}$ \\
      & 0.20 & $8.11_{\pm0.29}$ & $4.51_{\pm0.21}$ & $76.27_{\pm0.70}$
            & $5.57_{\pm0.05}$ & $3.88_{\pm0.12}$ & $76.01_{\pm0.15}$
            & $3.42_{\pm0.05}$ & $4.61_{\pm0.01}$ & $68.33_{\pm1.84}$
            & $8.80_{\pm0.07}$ & $6.80_{\pm0.02}$ & $88.17_{\pm0.58}$ \\
      & 0.25 & $8.03_{\pm0.06}$ & $4.50_{\pm0.21}$ & $75.40_{\pm1.60}$
            & $5.57_{\pm0.04}$ & $3.87_{\pm0.14}$ & $75.13_{\pm0.26}$
            & $3.41_{\pm0.02}$ & $4.63_{\pm0.03}$ & $68.79_{\pm1.60}$
            & $8.88_{\pm0.05}$ & $6.72_{\pm0.07}$ & $88.17_{\pm0.58}$ \\
    \midrule
    \multirow{5}{*}{\shortstack{Typical Acceptance~\cite{cai2024medusa}}}
      & 0.05 & $8.49_{\pm0.01}$ & $4.67_{\pm1.93}$ & $76.60_{\pm2.08}$
            & $5.64_{\pm0.03}$ & $3.89_{\pm0.14}$ & $76.46_{\pm1.06}$
            & $3.58_{\pm0.03}$ & $4.62_{\pm0.01}$ & $66.52_{\pm1.31}$
            & $8.87_{\pm0.04}$ & $6.84_{\pm0.02}$ & $89.00_{\pm0.50}$ \\
      & 0.10 & $8.47_{\pm0.05}$ & $4.66_{\pm0.00}$ & $75.20_{\pm1.11}$
            & $5.64_{\pm0.03}$ & $3.89_{\pm0.13}$ & $76.46_{\pm1.06}$
            & $3.63_{\pm0.03}$ & $4.61_{\pm0.01}$ & $66.52_{\pm1.05}$
            & $8.87_{\pm0.04}$ & $6.85_{\pm0.01}$ & $89.00_{\pm0.50}$ \\
      & 0.15 & $8.21_{\pm0.24}$ & $4.65_{\pm0.03}$ & $75.73_{\pm2.12}$
            & $5.66_{\pm0.03}$ & $3.88_{\pm0.14}$ & $75.66_{\pm0.53}$
            & $3.59_{\pm0.01}$ & $4.62_{\pm0.02}$ & $66.06_{\pm1.31}$
            & $8.87_{\pm0.04}$ & $6.85_{\pm0.03}$ & $89.00_{\pm0.50}$ \\
      & 0.20 & $8.39_{\pm0.01}$ & $4.65_{\pm0.00}$ & $74.67_{\pm1.55}$
            & $5.63_{\pm0.05}$ & $3.90_{\pm0.14}$ & $75.40_{\pm0.26}$
            & $3.54_{\pm0.03}$ & $4.60_{\pm0.01}$ & $68.79_{\pm0.69}$
            & $8.87_{\pm0.04}$ & $6.84_{\pm0.02}$ & $89.00_{\pm0.50}$ \\
      & 0.25 & $8.30_{\pm0.05}$ & $4.67_{\pm0.01}$ & $75.53_{\pm0.31}$
            & $5.55_{\pm0.08}$ & $3.90_{\pm0.14}$ & $75.04_{\pm0.15}$
            & $3.46_{\pm0.01}$ & $4.63_{\pm0.01}$ & $66.06_{\pm0.26}$
            & $8.89_{\pm0.01}$ & $6.78_{\pm0.04}$ & $88.67_{\pm0.29}$ \\
    \midrule
    \multirow{4}{*}{\shortstack{Lenience~\cite{leviathan2023fast}}}
      & 0.2 & $8.49_{\pm0.02}$ & $4.68_{\pm0.05}$ & $74.47_{\pm0.58}$
            & $5.65_{\pm0.03}$ & $4.06_{\pm0.02}$ & $75.13_{\pm0.26}$
            & $3.59_{\pm0.04}$ & $4.52_{\pm0.04}$ & $67.42_{\pm1.39}$
            & $8.87_{\pm0.01}$ & $6.90_{\pm0.02}$ & $88.67_{\pm0.29}$ \\
      & 0.4 & $8.37_{\pm0.08}$ & $4.66_{\pm0.12}$ & $75.60_{\pm2.65}$
            & $5.61_{\pm0.02}$ & $4.06_{\pm0.02}$ & $75.49_{\pm0.67}$
            & $3.55_{\pm0.06}$ & $4.53_{\pm0.09}$ & $67.88_{\pm1.31}$
            & $8.85_{\pm0.04}$ & $6.87_{\pm0.01}$ & $88.00_{\pm0.50}$ \\
      & 0.6 & $8.26_{\pm0.05}$ & $4.66_{\pm0.14}$ & $78.00_{\pm1.73}$
            & $5.52_{\pm0.07}$ & $4.06_{\pm0.03}$ & $75.57_{\pm0.40}$
            & $3.48_{\pm0.02}$ & $4.50_{\pm0.18}$ & $68.64_{\pm1.82}$
            & $8.83_{\pm0.02}$ & $6.84_{\pm0.02}$ & $88.17_{\pm0.29}$ \\
      & 0.8 & $8.22_{\pm0.07}$ & $4.65_{\pm0.23}$ & $76.47_{\pm1.15}$
            & $5.52_{\pm0.04}$ & $4.06_{\pm0.02}$ & $75.75_{\pm0.40}$
            & $3.38_{\pm0.03}$ & $4.49_{\pm0.11}$ & $68.33_{\pm0.26}$
            & $8.81_{\pm0.05}$ & $6.87_{\pm0.02}$ & $87.67_{\pm0.29}$ \\
    \midrule
    \multirow{4}{*}{\shortstack{CoS~\cite{fu2025fastlargelanguagemodel}}}
      & 0.2 & $8.33_{\pm0.00}$ & $4.68_{\pm0.01}$ & $71.53_{\pm2.12}$
            & $6.05_{\pm0.13}$ & $4.10_{\pm0.01}$ & $71.43_{\pm1.06}$
            & $3.89_{\pm0.04}$ & $4.61_{\pm0.02}$ & $64.55_{\pm1.98}$
            & $8.76_{\pm0.13}$ & $6.79_{\pm0.04}$ & $70.67_{\pm3.40}$ \\
      & 0.4 & $8.80_{\pm0.04}$ & $4.67_{\pm0.00}$ & $60.40_{\pm2.84}$
            & $7.14_{\pm0.09}$ & $4.13_{\pm0.01}$ & $61.20_{\pm3.29}$
            & $4.69_{\pm0.07}$ & $4.58_{\pm0.03}$ & $53.79_{\pm2.10}$
            & $8.84_{\pm0.15}$ & $6.69_{\pm0.03}$ & $58.00_{\pm4.50}$ \\
      & 0.6 & $9.42_{\pm0.01}$ & $4.67_{\pm0.01}$ & $54.80_{\pm0.72}$
            & $7.86_{\pm0.02}$ & $4.20_{\pm0.02}$ & $59.44_{\pm0.85}$
            & $5.99_{\pm0.03}$ & $4.53_{\pm0.01}$ & $48.03_{\pm1.89}$
            & $9.29_{\pm0.16}$ & $6.69_{\pm0.03}$ & $43.83_{\pm3.79}$ \\
      & 0.8 & $10.14_{\pm0.11}$ & $4.65_{\pm0.01}$ & $46.00_{\pm2.12}$
            & $9.26_{\pm0.04}$ & $4.26_{\pm0.00}$ & $54.06_{\pm0.31}$
            & $8.02_{\pm0.05}$ & $4.49_{\pm0.01}$ & $37.12_{\pm2.05}$
            & $10.26_{\pm0.07}$ & $6.75_{\pm0.01}$ & $43.67_{\pm1.89}$ \\
    \midrule
\multirow{4}{*}{Auto Judge~\cite{garipov2025autojudge}}
  & 0.2
  & $9.22_{\pm0.02}$ & $4.67_{\pm0.12}$ & $85.80_{\pm0.20}$
  & $5.89_{\pm0.03}$ & $4.09_{\pm0.02}$ & $76.19_{\pm0.26}$
  & $4.11_{\pm0.11}$ & $4.58_{\pm0.03}$ & $60.45_{\pm0.45}$
  & $8.79_{\pm0.05}$ & $6.81_{\pm0.02}$ & $84.50_{\pm0.50}$ \\
  & 0.4
  & $9.71_{\pm0.01}$ & $4.65_{\pm0.02}$ & $82.80_{\pm0.20}$
  & $5.93_{\pm0.04}$ & $4.08_{\pm0.01}$ & $73.54_{\pm0.26}$
  & $4.84_{\pm0.07}$ & $4.57_{\pm0.05}$ & $60.00_{\pm0.45}$
  & $8.65_{\pm0.04}$ & $6.78_{\pm0.03}$ & $78.50_{\pm0.50}$ \\
  & 0.6
  & $10.34_{\pm0.11}$ & $4.69_{\pm0.03}$ & $75.40_{\pm0.40}$
  & $6.22_{\pm0.09}$ & $4.07_{\pm0.03}$ & $70.11_{\pm0.53}$
  & $6.69_{\pm0.08}$ & $4.55_{\pm0.02}$ & $50.00_{\pm0.91}$
  & $8.95_{\pm0.10}$ & $6.76_{\pm0.02}$ & $66.00_{\pm0.50}$ \\
  & 0.8
  & $10.84_{\pm0.13}$ & $4.64_{\pm0.01}$ & $65.20_{\pm0.40}$
  & $8.32_{\pm0.03}$ & $4.06_{\pm0.02}$ & $60.58_{\pm0.53}$
  & $9.07_{\pm0.05}$ & $4.52_{\pm0.04}$ & $38.64_{\pm0.91}$
  & $9.65_{\pm0.07}$ & $6.74_{\pm0.01}$ & $47.00_{\pm1.00}$ \\

    \bottomrule[1.3pt]
  \end{tabular}}
  \label{tab:verification_result_qw25}
\end{table*}

\begin{table*}[t]
  \centering
  \caption{
    Verification results across four benchmarks: MATH, MBPP+, INCLUDE, and
    BFCL. The target and draft models are
    \texttt{meta-llama/Llama-3.1-8B-Instruct} and
    \texttt{meta-llama/Llama-3.2-1B-Instruct}, respectively.
    \textbf{Bold} denotes the best value in each column, and
    \underline{underline} denotes the second-best distinct value.
  }
  \resizebox{\linewidth}{!}{%
  \begin{tabular}{l|c|ccc|ccc|ccc|ccc}
    \toprule
    \multirow{2}{*}{\textbf{Method}}
      & \multirow{2}{*}{\textbf{Param}}
      & \multicolumn{3}{c|}{\textbf{MATH}}
      & \multicolumn{3}{c|}{\textbf{MBPP+}}
      & \multicolumn{3}{c|}{\textbf{INCLUDE}}
      & \multicolumn{3}{c}{\textbf{BFCL}} \\
    \cmidrule(lr){3-5}
    \cmidrule(lr){6-8}
    \cmidrule(lr){9-11}
    \cmidrule(lr){12-14}
      & & \textbf{BE} & \textbf{DS} & \textbf{Acc\,(\%)}
      & \textbf{BE} & \textbf{DS} & \textbf{Pass@1\,(\%)}
      & \textbf{BE} & \textbf{DS} & \textbf{Acc\,(\%)}
      & \textbf{BE} & \textbf{DS} & \textbf{Acc\,(\%)} \\
    \midrule
SD Baseline~\cite{leviathan2023fast} & --
      & 7.38 & 67.34 & 74.00
      & 5.76 & 61.36 & 57.67
      & 3.50 & 54.38 & 37.27
      & 7.06 & 40.47 & 80.50 \\
      \midrule

    \multirow{5}{*}{\shortstack{
      Min-$p$ Sampling~\citep{nguyen2025turningheatminpsampling}\\
      + SD}}
      & 0.1
      & 7.55 & 70.64 & 74.20
      & 6.00 & 63.38 & 60.32
      & 4.07 & 69.89 & \underline{43.18}
      & 7.37 & 41.64 & 80.50 \\
      & 0.3
      & 7.62 & 71.19 & 76.80
      & 5.96 & 64.30 & 58.99
      & 4.57 & 70.20 & 37.73
      & 7.34 & 41.03 & 81.50 \\
      & 0.5
      & 7.65 & 71.31 & 74.20
      & 6.00 & 64.15 & 58.20
      & 4.55 & 70.29 & \textbf{44.09}
      & 7.20 & 40.89 & 83.00 \\
      & 0.7
      & 7.67 & 71.62 & 77.00
      & 6.10 & 63.91 & \textbf{60.85}
      & 4.64 & 70.25 & 41.36
      & 7.27 & 41.36 & \textbf{84.00} \\
      & 0.9
      & 7.77 & 71.14 & 75.80
      & 6.18 & 63.31 & \underline{60.58}
      & 4.92 & 69.79 & 42.27
      & 7.36 & 40.41 & \underline{83.50} \\

    \midrule

    \multirow{5}{*}{SpecCascade~\cite{narasimhan2024faster}}
      & 0.1
      & 9.23 & 72.22 & 72.00
      & 6.96 & 64.37 & 59.26
      & 5.79 & 71.35 & 36.82
      & 7.74 & 40.95 & 81.00 \\
      & 0.3
      & 8.27 & 72.19 & 75.20
      & 6.26 & 64.50 & 60.32
      & 4.70 & 71.53 & 36.82
      & 7.50 & 40.49 & 83.00 \\
      & 0.5
      & 7.71 & 73.13 & 76.80
      & 6.14 & 65.16 & 59.26
      & 4.60 & 71.33 & 40.45
      & 7.36 & 41.90 & 81.00 \\
      & 0.7
      & 7.30 & 72.85 & 76.60
      & 5.93 & 65.31 & 59.79
      & 4.44 & 71.19 & 37.73
      & 7.31 & 41.04 & 82.00 \\
      & 0.9
      & 6.96 & 73.19 & 76.20
      & 5.91 & 65.31 & 59.79
      & 4.13 & 71.78 & 38.18
      & 7.24 & 41.92 & 77.50 \\

    \midrule

\multirow{5}{*}{\shortstack{$\eta$ Smpl.~\citep{hewitt2022truncation} \\ + SD}}
      & 0.05
      & 9.09 & 84.16 & 71.20
      & 6.79 & 70.08 & 58.20
      & 4.75 & 73.42 & 37.73
      & 7.37 & 40.12 & 79.00 \\
      & 0.10
      & 8.92 & 82.36 & 70.60
      & 6.71 & 71.30 & 58.99
      & 4.63 & 72.59 & 34.55
      & 7.28 & 39.67 & 78.50 \\
      & 0.15
      & 8.52 & 78.46 & 76.60
      & 6.36 & 66.13 & 58.20
      & 4.35 & 64.32 & 35.00
      & 7.21 & 38.27 & 81.00 \\
      & 0.20
      & 8.19 & 74.81 & 76.40
      & 6.29 & 65.40 & 58.99
      & 4.25 & 64.10 & 32.73
      & 7.24 & 40.24 & 81.50 \\
      & 0.25
      & 8.00 & 73.05 & 75.20
      & 6.10 & 62.91 & \underline{60.58}
      & 4.18 & 63.73 & 34.09
      & 7.19 & 40.06 & 82.50 \\

    \midrule

    \multirow{5}{*}{Typical Acceptance~\cite{cai2024medusa}}
      & 0.05
      & 7.51 & 69.34 & 75.60
      & 5.96 & 62.26 & 57.67
      & 4.04 & 61.41 & 39.55
      & 7.41 & 40.42 & 80.00 \\
      & 0.10
      & 7.53 & 68.73 & 72.00
      & 6.00 & 61.88 & 58.20
      & 4.10 & 63.53 & 38.64
      & 7.32 & 41.92 & 80.50 \\
      & 0.15
      & 7.68 & 72.12 & 75.60
      & 6.01 & 61.98 & \underline{60.58}
      & 4.29 & 63.42 & 39.09
      & 7.31 & 38.86 & 81.00 \\
      & 0.20
      & 7.66 & 71.93 & 75.40
      & 5.96 & 61.46 & 60.32
      & 4.28 & 64.27 & 32.73
      & 7.33 & 42.97 & 82.00 \\
      & 0.25
      & 7.66 & 72.93 & 74.20
      & 5.99 & 60.78 & 58.99
      & 4.34 & 65.19 & 34.09
      & 7.30 & 39.81 & \textbf{84.00} \\

    \midrule
    \multirow{4}{*}{Lenience~\cite{leviathan2023fast}}
      & 0.2
      & 9.02 & 83.33 & 72.20
      & 6.49 & 67.02 & 57.67
      & 4.58 & 70.85 & 36.36
      & 7.36 & 40.14 & 80.50 \\
      & 0.4
      & 8.74 & 88.93 & 74.40
      & 6.43 & 69.39 & \underline{60.58}
      & 4.40 & 68.20 & 35.45
      & 7.26 & 40.98 & 81.00 \\
      & 0.6
      & 8.45 & 78.05 & \textbf{79.20}
      & 6.25 & 65.39 & \textbf{60.85}
      & 3.97 & 60.24 & 34.55
      & 7.18 & 39.14 & 80.00 \\
      & 0.8
      & 7.94 & 72.52 & \underline{77.40}
      & 6.08 & 64.72 & \textbf{60.85}
      & 3.85 & 59.71 & 33.18
      & 7.13 & 38.86 & 83.00 \\
          \midrule

    \multirow{4}{*}{CoS~\cite{fu2025fastlargelanguagemodel}}
      & 0.2
      & 8.00 & 73.08 & 73.40
      & 6.55 & 67.64 & 52.91
      & 4.76 & 71.59 & 34.55
      & 7.63 & 42.45 & 61.00 \\
      & 0.4
      & 8.74 & 80.73 & 64.20
      & 7.48 & 77.39 & 51.06
      & 6.08 & 91.72 & 33.64
      & 8.26 & 45.17 & 50.00 \\
      & 0.6
      & 9.53 & 87.25 & 61.40
      & 8.47 & 89.77 & 45.50
      & 7.50 & 112.38 & 28.64
      & 9.10 & 49.87 & 39.50 \\
      & 0.8
      & \underline{10.27} & \underline{94.59} & 57.00
      & 9.74 & 101.28 & 42.06
      & \underline{9.29} & \underline{138.47} & 29.55
      & 10.03 & 55.27 & 38.50 \\
    \midrule

    \multirow{5}{*}{Auto Judge~\cite{garipov2025autojudge}}
      & 0.2
      & 8.61 & 73.94 & 77.00
      & 6.80 & 74.78 & 53.17
      & 8.84 & 67.68 & 30.00
      & 8.92 & 47.03 & 53.00 \\
      & 0.4
      & 9.89 & 71.86 & 67.60
      & 9.52 & 74.89 & 43.92
      & 10.80 & 67.36 & 25.00
      & 10.44 & 46.89 & 38.00 \\
      & 0.6
      & 10.67 & 72.22 & 56.60
      & 10.70 & 73.76 & 39.68
      & 10.91 & 66.71 & 25.00
      & 10.81 & 46.27 & 32.50 \\
      & 0.8
      & 10.97 & 72.36 & 53.00
      & 10.99 & 72.85 & 38.36
      & 10.92 & 66.99 & 28.64
      & 10.92 & 45.52 & 27.50 \\

    \bottomrule[1.3pt]
  \end{tabular}}
  \label{tab:verification_result_llama31}
\end{table*}

\begin{table*}[t]
\centering
\small
\caption{EAGLE-3 speculative decoding results across four benchmarks (LLaMA-3.1 8B, temperature $= 0.7$, block size$=7$).
\textbf{Bold} marks the best value per column across all methods;
\underline{underline} marks the second best.}
\label{tab:eagle_llama31}
\setlength{\tabcolsep}{4pt}
\renewcommand{\arraystretch}{1.15}
\resizebox{\linewidth}{!}{
\begin{tabular}{ll
  S[table-format=1.2]S[table-format=3.2]S[table-format=2.2]
  S[table-format=1.2]S[table-format=3.2]S[table-format=2.2]
  S[table-format=1.2]S[table-format=2.2]S[table-format=2.2]
  S[table-format=1.2]S[table-format=3.2]S[table-format=2.2]}
\toprule
& & \multicolumn{3}{c}{\textbf{MATH}} & \multicolumn{3}{c}{\textbf{MBPP+}} & \multicolumn{3}{c}{\textbf{INCLUDE}} & \multicolumn{3}{c}{\textbf{BFCL}} \\
\cmidrule(lr){3-5}\cmidrule(lr){6-8}\cmidrule(lr){9-11}\cmidrule(lr){12-14}
\textbf{Method} & \textbf{Param}
  & {\textbf{BE}} & {\textbf{DS}} & {\textbf{Acc (\%)}}
  & {\textbf{BE}} & {\textbf{DS}} & {\textbf{Pass@1 (\%)}}
  & {\textbf{BE}} & {\textbf{DS}} & {\textbf{Acc (\%)}}
  & {\textbf{BE}} & {\textbf{DS}} & {\textbf{Acc (\%)}} \\
\midrule
Baseline
 & --- & 3.76 & 140.21 & 73.20
 & 4.70 & 167.83 & 59.26
 & 0.67 & 45.72 & 35.45
 & 2.57 & 85.27 & 86.00 \\
\midrule
\multirow{4}{*}{Lenience~\cite{leviathan2023fast}}
 & 0.2 & 4.49 & 161.29          & 71.20          & 5.17 & {\underline{181.95}} & {\underline{61.11}} & 0.86 & 50.61          & 30.91 & 2.63 & {\textbf{87.31}} & 83.50 \\
 & 0.4 & 4.34 & 157.03          & 76.20          & 5.02 & 177.13               & 59.79               & 0.80 & 48.93          & 32.73 & 2.61 & {\underline{86.14}} & 85.50 \\
 & 0.6 & 4.19 & 152.85          & 76.00          & 4.96 & 175.43               & 59.52               & 0.73 & 47.03          & 34.55 & 2.59 & 85.77 & 86.00 \\
 & 0.8 & 4.07 & 149.39          & 75.20          & 4.92 & 174.25               & 59.79               & 0.72 & 46.73          & 33.64 & 2.59 & 85.56 & 84.00 \\
\midrule
\multirow{5}{*}{\shortstack{Min-$p$ Sampling~\citep{nguyen2025turningheatminpsampling}\\+ SD}}
 & 0.1 & 3.89 & 140.72 & 75.20 & 4.84 & 166.79 & 61.90 & 0.54 & 45.34 & {\underline{40.00}} & 2.58 & 84.45 & 86.00 \\
 & 0.3 & 3.96 & 142.85 & {\textbf{77.60}} & 4.89 & 168.29 & 60.32 & 0.55 & 45.65 & 36.82 & 2.58 & 84.13 & 86.00 \\
 & 0.5 & 3.98 & 143.42 & {\underline{77.40}} & 4.91 & 168.98 & 60.85 & 0.56 & 45.91 & 35.91 & 2.58 & 84.09 & 86.50 \\
 & 0.7 & {\underline{4.00}} & {\underline{143.88}} & {\underline{77.40}} & {\underline{4.92}} & {\underline{169.23}} & 62.43 & 0.56 & {\underline{45.96}} & 35.91 & 2.59 & 84.21 & {\textbf{87.50}} \\
 & 0.9 & {\textbf{4.05}} & {\textbf{145.24}} & 76.80 & {\textbf{4.93}} & {\textbf{169.58}} & 63.49 & {\textbf{0.57}} & {\textbf{46.13}} & 36.36 & 2.59 & 84.42 & {\underline{87.00}} \\
\midrule
\multirow{5}{*}{SpecCascade~\cite{narasimhan2024faster}}
 & 0.1 & 4.47 & 159.83          & 73.00          & 5.16 & 180.81               & 59.79               & 0.92 & 51.76          & 30.91 & 2.62 & 85.82 & 84.50 \\
 & 0.3 & 4.28 & 154.42          & {\underline{77.40}} & 5.03 & 176.79          & {\textbf{61.38}}    & 0.79 & 48.26          & 31.82 & 2.59 & 85.18 & 84.50 \\
 & 0.5 & 4.18 & 151.56          & {\textbf{77.60}} & 4.99 & 175.66             & 60.85               & 0.76 & 47.43          & 31.82 & 2.58 & 85.03 & 86.50 \\
 & 0.7 & 4.12 & 149.71          & 76.20          & 4.93 & 173.70               & {\textbf{61.38}}    & 0.73 & 46.74          & {\textbf{36.82}} & 2.58 & 84.99 & 85.50 \\
 & 0.9 & 4.06 & 148.20          & 76.60          & 4.89 & 172.60               & 60.32               & 0.72 & 46.45          & 35.00 & 2.58 & 85.00 & 86.00 \\
\midrule

\multirow{5}{*}{\shortstack{$\eta$ Smpl.~\citep{hewitt2022truncation} \\ + SD}}
 & 0.05 & 3.85 & 137.92 & 75.20 & 4.84 & 164.79 & 59.79 & 0.55 & 45.42 & 35.00 & 2.58 & 83.63 & {\textbf{87.50}} \\
 & 0.10 & 3.90 & 139.33 & 77.00 & 4.86 & 165.32 & 61.11 & 0.54 & 45.06 & {\textbf{41.82}} & 2.59 & 83.92 & 86.00 \\
 & 0.15 & 3.93 & 140.21 & 77.20 & 4.88 & 165.95 & 61.90 & 0.54 & 45.13 & 35.00 & 2.59 & 83.91 & 86.50 \\
 & 0.20 & 3.96 & 140.97 & 76.00 & 4.89 & 166.23 & {\textbf{64.29}} & 0.54 & 45.14 & 34.09 & 2.57 & 83.54 & 85.00 \\
& 0.25 & 3.99 & 141.83 & 75.20
& 4.88 & 165.77 & 63.76
& 0.57 & 45.92 & 37.73
& 2.58 & 83.85 & 86.50 \\
\midrule
    \multirow{5}{*}{\shortstack{Typical Acceptance~\cite{cai2024medusa}}}
 & 0.05 & {\textbf{4.82}} & {\textbf{168.77}} & 66.00 & {\textbf{5.29}} & {\textbf{182.86}} & 55.29 & {\textbf{1.24}} & {\textbf{59.80}} & 29.55 & {\textbf{2.66}} & 86.11 & 78.00 \\
 & 0.10 & {\underline{4.66}} & {\underline{164.19}} & 68.20 & {\underline{5.21}} & 180.54 & 56.08 & {\underline{1.15}} & {\underline{57.20}} & 29.55 & {\underline{2.65}} & 85.85 & 79.50 \\
 & 0.15 & 4.58 & 161.86          & 72.20          & 5.16 & 178.87               & 57.41               & 1.05 & 54.76          & 25.91 & 2.64 & 85.65 & 83.00 \\
 & 0.20 & 4.54 & 160.65          & 69.80          & 5.16 & 178.94               & 58.47               & 0.98 & 52.73          & 27.73 & 2.63 & 85.40 & 83.50 \\
 & 0.25 & 4.46 & 158.17          & 73.00          & 5.17 & 179.31               & 57.14               & 1.00 & 53.21          & 26.82 & 2.62 & 85.05 & 83.00 \\
\bottomrule
\end{tabular}}
\end{table*}

\begin{table*}[t]
  \centering
\caption{EAGLE-3 speculative decoding results across four benchmarks (Qwen3-8B, temperature $= 0.7$, block size$=7$).
\textbf{Bold} marks the best value per column across all methods;
\underline{underline} marks the second best.}
  \resizebox{\linewidth}{!}{%
  \begin{tabular}{l|c|ccc|ccc|ccc|ccc}
    \toprule
    \multirow{2}{*}{\textbf{Method}}
      & \multirow{2}{*}{\textbf{Param}}
      & \multicolumn{3}{c|}{\textbf{MATH}}
      & \multicolumn{3}{c|}{\textbf{MBPP+}}
      & \multicolumn{3}{c|}{\textbf{INCLUDE}}
      & \multicolumn{3}{c}{\textbf{BFCL}} \\
    \cmidrule(lr){3-5}
    \cmidrule(lr){6-8}
    \cmidrule(lr){9-11}
    \cmidrule(lr){12-14}
      & & \textbf{BE} & \textbf{DS} & \textbf{Acc\,(\%)}
      & \textbf{BE} & \textbf{DS} & \textbf{Pass@1\,(\%)}
      & \textbf{BE} & \textbf{DS} & \textbf{Acc\,(\%)}
      & \textbf{BE} & \textbf{DS} & \textbf{Acc\,(\%)} \\
    \midrule
Baseline
      & ---
      & 4.54 & 162.99 & 81.60
      & 4.68 & 163.24 & 51.32
      & 1.81 & 117.06 & 46.36
      & 2.94 & 95.97 & 81.50 \\
 \midrule
    \multirow{4}{*}{Lenience~\cite{leviathan2023fast}}
      & 0.2
      & 4.92 & \underline{177.76} & 82.80
      & \underline{5.19} & 179.13 & 48.68
      & 1.96 & 128.99 & 46.82
      & 2.99 & 97.64 & 79.50 \\
      & 0.4
      & 4.84 & 175.13 & 83.20
      & 5.01 & 174.05 & 51.85
      & 1.96 & 126.66 & 47.73
      & 2.98 & 97.83 & 80.50 \\
      & 0.6
      & 4.74 & 170.87 & 81.20
      & 4.93 & 171.24 & 53.97
      & 1.87 & 124.63 & 44.09
      & 2.97 & 96.88 & 79.00 \\
      & 0.8
      & 4.69 & 169.67 & 85.20
      & 4.83 & 166.80 & 52.38
      & 1.86 & 117.87 & 44.55
      & 2.95 & 96.50 & 82.00 \\
    \midrule
    
\multirow{5}{*}{\shortstack{Min-$p$ Sampling~\citep{nguyen2025turningheatminpsampling}\\+ SD}}
      & 0.1
      & 4.60 & 166.01 & 84.80
      & 4.75 & 165.56 & 53.97
      & 1.90 & 127.20 & 45.91
      & 2.96 & 97.14 & 85.00 \\
      & 0.3
      & 4.65 & 168.24 & 84.60
      & 4.75 & 164.18 & 52.12
      & 1.81 & 123.75 & 47.27
      & 2.96 & 96.80 & \underline{85.50} \\
      & 0.5
      & 4.66 & 167.17 & \underline{85.60}
      & 4.81 & 167.42 & 53.17
      & 1.86 & 123.60 & 41.36
      & 2.96 & 96.47 & \underline{85.50} \\
      & 0.7
      & 4.64 & 165.93 & 84.80
      & 4.77 & 164.67 & 55.82
      & 1.88 & 133.99 & 45.00
      & 2.96 & 96.50 & \underline{85.50} \\
      & 0.9
      & 4.66 & 166.33 & 85.40
      & 4.75 & 165.85 & 55.82
      & 1.87 & 137.57 & 43.64
      & 2.96 & 97.15 & \underline{85.50} \\

    \midrule
    \multirow{5}{*}{SpecCascade~\cite{narasimhan2024faster}}
      & 0.1
      & \underline{4.94} & 177.01 & 83.00
      & \underline{5.19} & \underline{179.58} & 49.74
      & 1.97 & 127.06 & 43.64
      & 2.98 & 97.31 & 80.50 \\
      & 0.3
      & 4.77 & 172.16 & 83.40
      & 4.95 & 171.89 & 53.17
      & 1.92 & 137.64 & 44.09
      & 2.97 & 97.45 & 79.50 \\
      & 0.5
      & 4.72 & 169.82 & 84.20
      & 4.82 & 168.04 & 53.97
      & 1.89 & 121.72 & 45.00
      & 2.96 & 96.80 & 79.00 \\
      & 0.7
      & 4.69 & 169.53 & 85.20
      & 4.78 & 167.23 & 52.65
      & 1.83 & 133.27 & 47.27
      & 2.95 & 96.82 & 80.00 \\
      & 0.9
      & 4.66 & 168.22 & 83.60
      & 4.75 & 164.58 & 57.14
      & 1.85 & 125.19 & 43.18
      & 2.94 & 96.36 & 78.50 \\

    \midrule

\multirow{5}{*}{\shortstack{$\eta$ Smpl.~\citep{hewitt2022truncation} \\ + SD}}
      & 0.05
      & 4.59 & 165.41 & 84.40
      & 4.69 & 162.52 & 54.50
      & 1.80 & 129.36 & 50.00
      & 2.95 & 96.80 & 84.50 \\
      & 0.10
      & 4.63 & 165.44 & 84.60
      & 4.76 & 166.15 & 55.82
      & 1.82 & 126.64 & 51.36
      & 2.96 & 96.79 & \underline{85.50} \\
      & 0.15
      & 4.64 & 167.65 & \textbf{86.00}
      & 4.73 & 164.01 & 54.76
      & 1.97 & 140.23 & 49.55
      & 2.96 & 96.61 & 84.50 \\
      & 0.20
      & 4.64 & 167.49 & 84.40
      & 4.73 & 164.25 & 57.41
      & 1.86 & 130.50 & 48.18
      & 2.96 & 96.54 & \textbf{86.00} \\
      & 0.25
      & 4.64 & 167.67 & 85.00
      & 4.79 & 167.47 & 50.79
      & 1.84 & 118.36 & 43.64
      & 2.95 & 96.86 & 84.50 \\

    \midrule

    \multirow{5}{*}{Typical Acceptance~\cite{cai2024medusa}}
      & 0.05
      & \textbf{4.98} & \textbf{179.56} & 80.40
      & \textbf{5.23} & \textbf{181.43} & 49.21
      & \underline{2.29} & \textbf{165.55} & 46.82
      & 2.98 & 97.49 & 80.50 \\
      & 0.10
      & 4.87 & 175.36 & 83.40
      & 5.16 & 178.41 & 49.21
      & \textbf{2.39} & \underline{150.90} & \textbf{57.73}
      & \textbf{3.03} & \textbf{99.11} & 79.00 \\
      & 0.15
      & 4.86 & 175.56 & 85.20
      & 4.90 & 171.26 & \textbf{70.11}
      & 2.19 & 143.56 & \underline{52.73}
      & \underline{3.02} & \underline{98.39} & 79.00 \\
      & 0.20
      & 4.81 & 173.48 & 84.60
      & 5.07 & 176.31 & 66.67
      & 1.90 & 118.54 & 50.45
      & 2.94 & 96.28 & 84.50 \\
      & 0.25
      & 4.80 & 173.41 & 84.00
      & 4.93 & 170.26 & \underline{69.31}
      & 1.90 & 138.16 & 46.82
      & 3.01 & 98.32 & 77.00 \\

    \bottomrule[1.3pt]
  \end{tabular}}
  \label{tab:eagle_qwen3}
\end{table*}

\section{Use of AI Assistants.}
We used OpenAI Codex and Anthropic Claude to assist with language editing, code development, and debugging. All AI-generated text and code suggestions were reviewed, revised, and verified by the authors. The authors remain fully responsible for the implementation, experimental results, analysis, and final manuscript.

\end{document}